\documentclass[final]{siamart220329}
 \usepackage{mathptmx}      % use Times fonts if available on your TeX system
\usepackage{amsmath, amssymb, bm}
\usepackage{graphicx}
\usepackage{hyperref}
\usepackage{longtable}
\usepackage{multicol}
\usepackage{tcolorbox}
\usepackage{subcaption}
\usepackage{algorithm}
\usepackage{algpseudocode}
\usepackage{geometry}

\newtheorem{example}{Example}
\newtheorem{theoretical insights}{Theoretical insights}

\newtheorem{practical takeaway}{Practical takeaway}

\DeclareMathAlphabet{\mathcal}{OMS}{cmsy}{m}{n}

\usepackage{lipsum}

\usepackage{epstopdf}  % Optional: Only needed for EPS files
\ifpdf
  \DeclareGraphicsExtensions{.pdf,.png,.jpg}
\else
  \DeclareGraphicsExtensions{.eps}
\fi

\newsiamremark{remark}{Remark}
\newsiamremark{hypothesis}{Hypothesis}
\crefname{hypothesis}{Hypothesis}{Hypotheses}
\newsiamthm{claim}{Claim}

\title{{Objective Value Change and Shape-Based Accelerated Optimization for the Neural Network Approximation}}

\author{
Pengcheng Xie\thanks{Applied Mathematics and Computational Research Division, Lawrence Berkeley National Laboratory (\email{pxie@lbl.gov}). Corresponding author.} 
\and Zihao Zhou\thanks{School of Mathematics and Statistics, Xi'an Jiaotong University} 
\and Zijian Zhou\thanks{School of Mathematics,  Southwest Jiaotong University}
}

\usepackage{amsopn}

\usepackage{booktabs} 

\usepackage{amsmath}
\usepackage{tikz}
\usetikzlibrary{arrows.meta}

\begin{document}

\maketitle

\begin{abstract}  
This paper introduce a novel metric of an objective function \(f\), we say \texttt{VC} (value change) to measure the difficulty and approximation affection when conducting an neural network approximation task, and it numerically supports characterizing the local performance and behavior of neural network approximation. Neural networks often suffer from unpredictable local performance, which can hinder their reliability in critical applications. \texttt{VC} addresses this issue by providing a quantifiable measure of local value changes in network behavior, offering insights into the stability and performance for achieving the neural-network approximation. We investigate some fundamental theoretical properties of \texttt{VC} and identified two intriguing phenomena in neural network approximation: the \texttt{VC}-tendency and the minority-tendency. These trends respectively characterize how pointwise errors evolve in relation to the distribution of \texttt{VC} during the approximation process.In addition, we propose a novel metric based on \texttt{VC}, which measures the distance between two functions from the perspective of variation. Building upon this metric, we further propose a new preprocessing framework for neural network approximation. Numerical results including the real-world experiment and the PDE-related scientific problem support our discovery and pre-processing acceleration method.

\end{abstract}


\setcounter{tocdepth}{3} 
\tableofcontents

\section{Introduction}

Approximation techniques have always played a pivotal role in the development of numerical algorithms. These methods are essential when exact solutions are either computationally infeasible or difficult to obtain. Traditional approaches to approximation, such as polynomial interpolation, spline methods, and finite element analysis, have been widely used in various fields, including optimization \cite{powell2006newuoa,wild2008orbit,xie2023derivativecombine,xie2024derivativetransform,XIE2025116146,xieyuannew,Xie_Yuan_2022,xie2025remuregionalminimalupdating}, numerical integration, and differential equation solvers \cite{bishop2006pattern}.  
In recent years, the rapid advancement in computational power, particularly the rise of GPUs and parallel processing, has enabled more sophisticated approaches to approximation, notably through the application of neural networks \cite{zhou2022towards}. Neural networks, which were initially conceptualized over 70 years ago, are now at the forefront of modern approximation techniques \cite{hornik1989multilayer, wu2024fedbiot, guolq}. Their ability to model complex and highly nonlinear relationships between inputs and outputs has been proven to be particularly effective in addressing problems that are difficult to tackle using classical methods \cite{goodfellow2016deep}. The resurgence of neural networks, fueled by the emergence of deep learning architectures, has led to significant breakthroughs in various machine learning domains, such as computer vision, natural language processing, and reinforcement learning \cite{lecun2015deep}. More importantly, neural networks have found impact applications in numerical mathematics, particularly in areas like function approximation, optimization, and solving differential equations \cite{larson2019derivative}.  
Neural networks have become a cornerstone of industrial innovation, driving advancements across a wide array of applications \cite{krizhevsky2012imagenet, liu2021swin, yang2024review,NEURIPS2020_44e76e99}. In the realm of computer vision, convolutional neural networks (CNNs) are extensively used for tasks such as image classification, object detection, and quality inspection in manufacturing \cite{krizhevsky2012imagenet}. Similarly, in natural language processing, models based on the Transformer architecture have revolutionized applications like machine translation, sentiment analysis, and conversational AI \cite{liu2021swin}. These industrial applications highlight the effectiveness of neural networks in handling large-scale, complex data with remarkable accuracy and efficiency.

In scientific computing, neural networks are increasingly being used to solve traditional numerical problems with new levels of efficiency and generalization. Physics-informed neural networks (PINNs) \cite{raissi2019physics}, for example, embed partial differential equations into the loss function of neural networks, enabling the solution of complex physical systems without requiring discretized meshes. Additionally, neural surrogates are employed to accelerate simulations in fields such as fluid dynamics \cite{cai2021physics}, materials science \cite{xie2018crystal,noe2020machine}, and quantum mechanics \cite{carleo2017solving}. These developments demonstrate neural networks are not only tools for data-driven tasks but are also emerging as powerful alternatives or supplements to classical numerical solvers in computational science. 
In numerical optimization, neural networks have demonstrated superior performance by approximating the global minimum of non-convex functions, a notoriously difficult problem for classical algorithms \cite{mnih2015human}. Furthermore, methods like neural ordinary differential equations (Neural ODEs) \cite{Chen2018neural} and physics-informed neural networks \cite{raissi2019physics} represent the confluence of neural networks with numerical methods, allowing for the direct incorporation of physical laws and governing equations into the learning process, thus providing highly accurate approximations with reduced computational cost. %The continual development of neural network architectures and the improvement of optimization algorithms used for training them have contributed to the growing popularity and success of neural networks in numerical applications. As computing devices continue to evolve, enabling more efficient and large-scale computations, the role of neural networks in approximation and numerical analysis is expected to expand, offering novel solutions to complex problems that were once considered intractable \cite{goodfellow2016deep}.

Neural networks have also attracted significant theoretical attention, particularly regarding their capability as universal function approximators. Foundational research has shown that even a single hidden-layer neural network with a sufficient number of neurons can approximate any continuous function on compact subsets of $\mathbb{R}^n$, given appropriate activation functions \cite{cybenko1989approximation}. This universal approximation theorem underpins much of the theoretical motivation for deploying neural networks in diverse domains, and subsequent work has explored the expressive power and convergence properties of deep and wide network architectures \cite{leshno1993multilayer,telgarsky2016benefits,lu2017expressive,barron1993universal,yarotsky2017error}.
Another important theoretical framework is the neural tangent kernel (NTK) \cite{jacot2018neural}. NTK characterizes how a neural network behaves under gradient descent by linearizing its evolution around its initialization. In this regime, the network's training can be described using kernel methods, allowing for rigorous analysis of generalization and convergence properties. The NTK remains constant during training for infinitely wide networks, making it a tractable object of study and providing deep insights into the connection between neural networks and classical kernel methods.
Belkin's work \cite{liu2020linearity} suggests that the constancy of the NTK arises from a transition to linearity, where certain large-scale nonlinear systems become approximately linear within a fixed-size neighborhood around the initialization point.
In subsequent work, Wang applied the NTK theory to the solution of differential equations by neural networks (PINN) \cite{wang2022and}, and Zhou extended it to general nonlinear PDEs \cite{zhou2024neural}.

One of the main objectives of this work is to develop new approaches for characterizing the approximation process of neural networks and to uncover potential underlying patterns. In the approximation process, under translation equivalence, a neural network essentially needs only to approximate the shape of the objective function, a perspective we refer to as shape approximation. Xu et al. proposed the frequency principle \cite{xu2020frequency,xu2024overview}, which examines the approximation process of neural networks in the frequency domain. Their work demonstrates that neural networks consistently prioritize the approximation of the low-frequency components of the objective function—an important result that not only aligns with empirical observations but also corresponds to certain training heuristics, such as early stopping.

 {\bf Main contributions.}
 This paper presents significant advancements in neural network approximation by exploring a minority-tendency phenomenon and introducing a novel pre-processing technique based on value change (\texttt{VC}) theory. The primary contributions are outlined as follows:

 \begin{itemize}
     \item[-] \textcolor{black}{We found that the process of neural networks approximating functions is related to the local variation of the objective function, which we refer to as ``value change'' (\texttt{VC}).}

     \item[-] \textcolor{black}{A \texttt{VC}-tendency phenomenon is observed in neural network approximation, suggesting that \texttt{VC} dominates the approximation behavior under certain conditions.}

     \item[-] \textcolor{black}{From the perspective of the \texttt{VC} distribution function, we observe the minority-tendency phenomenon, where the approximation tends to prioritize regions corresponding to lower distribution values.}

     \item[-] \textcolor{black}{We propose a novel norm (\texttt{IVC} distance) based on VC theory that captures only the differences in local variations between two functions, independent of their specific function values.}

     \item[-] \textcolor{black}{We propose a novel preprocessing criterion for neural network training, designed to minimize the \texttt{IVC} distance between the neural network (after preprocessing) and the objective function, thereby facilitating faster convergence and improving overall learning efficiency.}

 \end{itemize}

The remainder of this paper is organized as follows. In Section~\ref{motvation}, we present the motivation behind our research, including insights derived from two simple linear function approximation processes. In Section~\ref{sec3}, we introduce an efficient framework for characterizing neural network approximation, referred to as (objective) value change (VC). Section~\ref{sec4} discusses two VC-related phenomena, namely the VC-tendency and the Minority-tendency. In Section~\ref{sec5}, we propose a novel VC-based norm, upon which we develop a neural network preprocessing method that significantly reduces training time and improves accuracy. Experimental results validating the effectiveness of this approach are presented in Section~\ref{sec6}.
Finally, Section~\ref{sec7} concludes the paper.

\section{\textcolor{black}{
Positive correlation between approximation and the slope}}\label{motvation}

{\color{black} In this section, we present several basic examples to reveal certain patterns in neural network approximation. These examples demonstrate that the approximation error of the neural network is directly related to the slope of the objective function.} 
In order to reveal some basic laws in the approximation process of neural networks, we consider the approximation problem of fully connected neural networks (FCNNs, also known as multilayer perceptron (MLP)) under the mean squared error (MSE) loss function:
\begin{equation}\label{MSE}
    \texttt{Loss}(\psi_{\rm NN})=\frac{1}{N}\sum_{i=1}^N\left|\psi_{\rm NN}({\bm x}_i)-f({\bm x}_i)\right|^2,
\end{equation}
where $\psi_{\rm NN}$ and $f$ represent neural network (approximation model) and objective function, respectively. $N$ in \eqref{MSE} denotes the number of elements in the training or test dataset. {The} standard choice for the FCNN with $L$ layers ($L-1$ hidden layers) is defined recursively as 
\begin{equation*}
\begin{aligned}
\bm{m}^{(0)}({\bm x})=\bm{W}^{(0)}\bm{x}+\bm{b}^{(0)},\ 
\bm{\tilde{m}}^{(i)}({\bm x})=\sigma(\bm{m}^{(i-1)}({\bm x}))\in\mathbb{R}^{N_i},\ 
\bm{m}^{(i)}({\bm x})=\bm{W}^{(i)}\bm{\tilde{m}}^{(i)}({\bm x})+\bm{b}^{(i)},
\end{aligned}
\end{equation*} 
$i=1,\dots, L$, where $\bm{W}^{(i)}(\in\mathbb{R}^{N_{i+1}\times N_{i}})$ and $\boldsymbol{b}^{(i)}(\in\mathbb{R}^{N_{i+1}})$ are the weight matrices to be trained, and $\sigma$ is a coordinate-wise activation function ($\sigma$ is usually chosen as ReLU, Sigmoid or Tanh. In the following numerical experiments, we use Tanh as the activation function). In the usual initialization of FCNN, all the weights and biases are initialized to be independent and identically distributed (i.i.d.) as uniform distribution $U(-\sqrt{k},\sqrt{k})$, where $k$ is the input width of the previous layer of neural networks.

\subsection{Approximation to linear functions}\label{eg_lin}

We begin by investigating simple linear functions using neural networks as a benchmark. Specifically, we examine a set of linear functions with varying slopes to evaluate how the proposed metric responds to different rates of change. The approximation results, along with the corresponding slope values (including the representative slope used in Example~\ref{eg2}), are presented below.

\begin{example}[Approximation to three-dimensional  linear functions] We delve into the influence of the variable slope on the neural network’s ability to approximate three-dimensional linear functions. This investigation aims to understand how changes in slope affect the accuracy and performance of the neural network model. The specific three-dimensional linear functions under consideration are defined as follows:
\begin{equation*}
    f_i(x, y, z)=\kappa x + \kappa y + \kappa z,\ \forall \ (x,y,z)^{\top} \in [-1,1]^3, \ \kappa \in\{1, 10\}.
\end{equation*}
\end{example}

\begin{figure}[htb]
\centering
\includegraphics[width=1.0\textwidth,trim=55 10 50 20,clip]{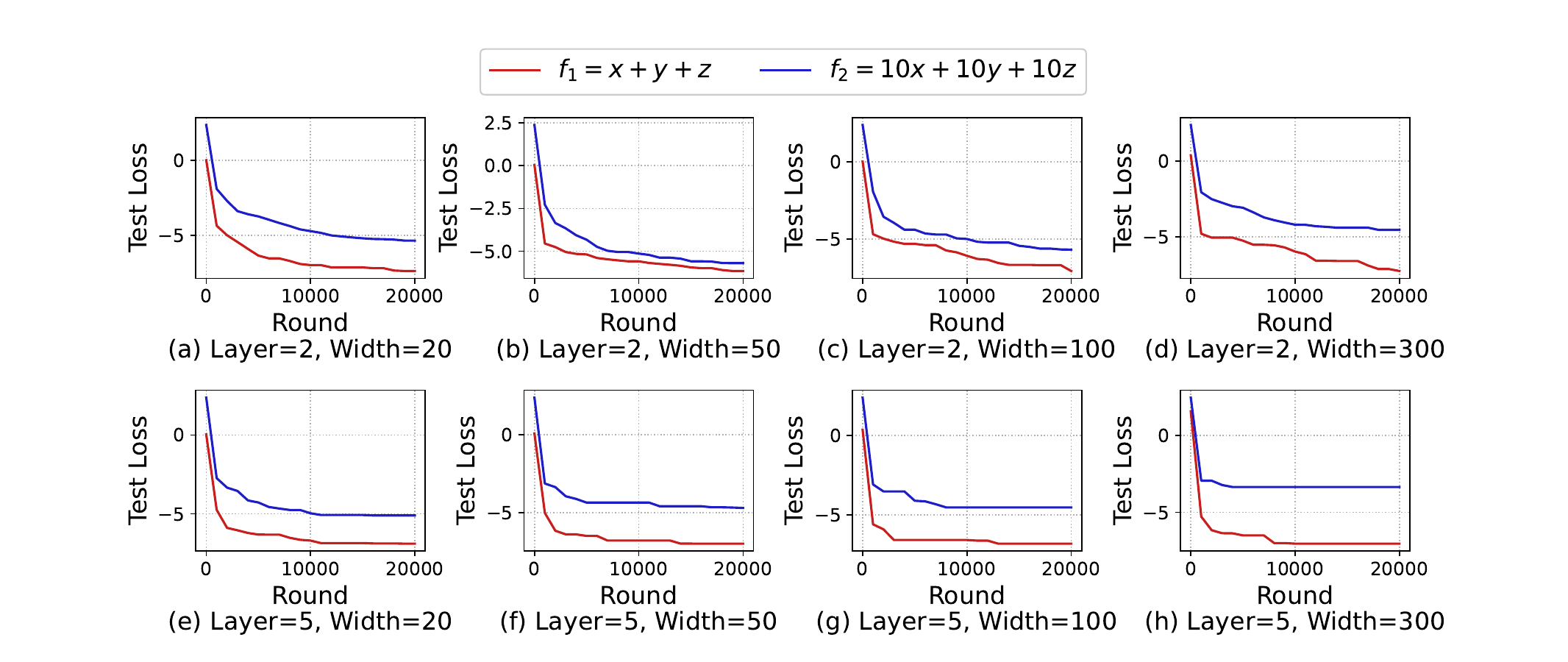}     
\caption{\color{black}{Correlation between linear function slope and approximation speed/efficiency.} Two three-dimensional linear functions with different slopes are approximated by a neural network over $[-1,1]^3$, using the \texttt{Adam} optimizer and mean squared error (MSE) loss. The learning rate is 
\(10^{-2}\).}
\label{fig:linear}
\end{figure}
\begin{sloppypar}
Fig.~\ref{fig:linear} exhibits the test loss curve of neural network approximation with 8 different network settings (Layer$\in\{2,5\}$, Width$\in\{20,50,100,300\}$). The results show that the linear function with a small slope converges faster to the well-trained model. 
{For example,} when ${\rm training\ steps} = 10,000$, the test loss for the $f_2$ approximation is always smaller than the $f_1$ approximation in these examples. In addition, neural network approximation experiments for one-dimensional or two-dimensional linear functions are available at an online repository\footnote{\href{https://github.com/pxie98/L-Change/tree/main/VC}{\ttfamily  https://github.com/pxie98/L-Change/tree/main/VC}}. The experimental results show that the slope of the objective function is negatively correlated with the convergence rate when the neural network approximates a simple (high-dimensional) linear function. 
\end{sloppypar}

\subsection{Approximation to piece-wise functions}

We study the difference in the approximation effect of different segments of a piecewise linear function when neural networks approximate segments with different slopes. We consider using the same preset neural network from Section \ref{eg_lin} to approximate a piece-wise linear function, where the slopes remain constant within each segment.

\begin{example}[Approximation to piece-wise functions]\label{eg2} The approximated objective  is defined in a symmetrical interval with different slopes on the left and right ends, as follows:
\begin{equation}\label{pw}
    f({x})=\left\{
    \begin{aligned}
    & 2x+2,\ \forall\ x\in[-2,0],\\
   &0{\ \ \rm or\ } -x+1,\ \forall\ x \in (0,2].
    \end{aligned}
    \right.
\end{equation}
Since piece-wise linear functions vary greatly at discontinuity points, and neural networks are continuous functions, we compare the test errors on $[-1.5, -0.5]$ and $[0.5, 1.5]$.

\end{example}

During the experiment, neural networks with Layers $=\{2,5\}$ and Width $=\{100,300\}$ were selected for two different right-hand slopes ($f_1=0$ and $f_2=-x+1$). 
Fig.~\ref{fig:piece-wise} shows the approximation results. It can be found that whether it is $f_1$ or $f_2$, under different network settings, the test loss on the right side decreases faster. This means that within the same function, a smaller slope causes the function to be approximated faster by neural networks.

\begin{figure}[htb]
\centering
\includegraphics[width=1.0\textwidth,trim=50 5 50 20,clip]{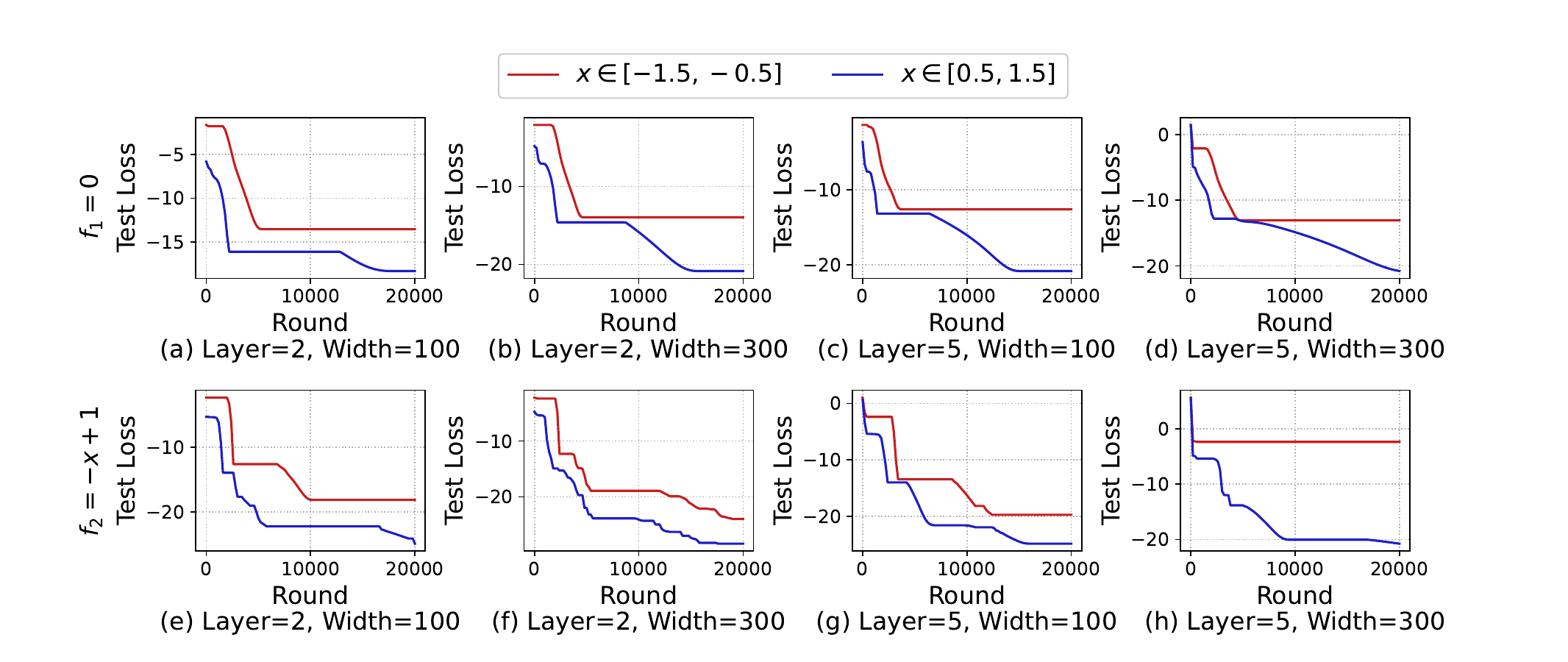} 
\caption{\color{black}{The correlation between different slopes and the approximation speed/efficiency within a single function.} The optimizer is \texttt{Adam} and the loss is MSE. Learning rate is \(10^{-2}\). } 
\label{fig:piece-wise}
\end{figure}

Fig.~\ref{fig:piece-wise} presents the loss curves of neural network approximation across 4 different network settings (Layer $\in\{2,5\}$, Width $\in\{20,100\}$). For instance, at 10,000 training steps, for both objective functions ($f_1$ and $f_2$) under different neural network parameter settings, the cases with smaller slopes converge faster numerically, which follows the truth that the objective with a smaller slope will be easier to be approximated by the neural network.  
The above two examples show that the magnitude of the slope of the objective function significantly affects its performance when approximated by neural networks, whether in the same function or in different functions. Although the slope can be directly derived from the objective function, for general neural networks approximation tasks (the objective function expression is unknown), there are often problems such as less data, and extremely complex objective function, and it is difficult to approximate the derivative. Therefore, the slope (derivative) is not a good indicator of how easily the objective function can be approximated. This prompts us to design a new metric to describe the ease with which the objective function is approached. This will be emphasized in the next section.

\section{\color{black}{Local measurement of function changes}}\label{sec3}

The study of how to characterize the local performance and behavior of neural network approximation is crucial in the theoretical research on deep learning. It is certain that the performance of the neural network approximation is related to the local change of the objective function. {The slope can reflect the local change of the objective function to a certain extent, but it has obvious locality and cannot reflect the influence of two close sampling points on each other when neural networks approach. Therefore, a suitable metric should be able to reflect the interaction between sampling points in the learning process under different circumstances.} 
To effectively quantify these changes, we introduce the concept of \texttt{VC}, a measure that captures the maximum variation of a function over a specified interval. This concept is not only central to understanding the behavior of functions but also plays a significant role in the performance of approximation algorithms. Along with \texttt{VC}, we define its density distribution to gain further insight into how these changes are spread across the domain of the function. Below, we provide formal definitions and explore their implications.

\subsection{\color{black}Value change}

We consider the real objective function \(f:\mathbb{R}^n \rightarrow\mathbb{R}\) in the neural network approximation problem in the interval \(\Omega=[a_1,b_1]\times\dots\times [a_n,b_n] \subset \mathbb{R}^n\), where \(a_i < b_i\), \(\forall \ i=1,\dots,n\). A similar definition is also applicable to other complex objective functions of \texttt{VC}.  

\begin{definition}[The \texttt{VC} of \(f\) at \({\bm x}\)]\label{VC}
For \(\forall\ {\bm x}=(x_1,\dots,x_n)^{\top} \in [a_1,b_1]\times\dots\times [a_n,b_n]\),  the \texttt{VC} of the objective function \(f\) at  \({\bm x}\) is defined as  
\begin{equation}
\texttt{VC}_{L}(f,{\bm x}) = \sup_{{\bm y}_1, {\bm y}_2 \in \Omega \cap \prod_{i=1}^{n} [{x}_i-\frac{L}{2}, {x}_i+\frac{L}{2}]} \ \vert f({\bm y}_1)-f({\bm y}_2)\vert.  
\end{equation}
\end{definition}

The relationship between \texttt{VC} and $L$ is closely related to $f$. 
Specifically, the following proposition holds (we take the one-dimensional case as example).

\begin{proposition}
If $f({x})\in\mathcal{C}^1$, Then $\frac{\partial\texttt{VC}_L(f,{x_0)}}{\partial L}\Big|_{L^+,L=0}=\frac{df(x_0)}{dx}$ ($L^+$ denotes the directional derivative in the $L^+$ direction). 
\end{proposition}

\begin{proof}
According to the definition of \texttt{VC} and the definition of the Taylor expansion, the conclusion regarding $f(x)\in\mathcal{C}^1$ can be  proven. As defined by \texttt{VC}, it is known that
\[\begin{aligned}
\frac{\texttt{VC}_L(f,x_0)-\texttt{VC}_0(f,x_0)}{L}
&=\frac{1}{L}\max_{y_1,y_2\in[x_0-\frac{L}{2},x_0+\frac{L}{2}]}|f(y_1)-f(y_2)|\\
&=\frac{1}{L}\max_{y_1,y_2\in[x_0-\frac{L}{2},x_0+\frac{L}{2}]}|f(x_0)-f(y_1)-(f(x_0)-f(y_2))|\\
&\leq\frac{2}{L}\max_{y\in[x_0-\frac{L}{2},x_0+\frac{L}{2}]}|f(x_0)-f(y)|.
\end{aligned}\]
Then, we have 
\[\begin{aligned}
\frac{\partial\texttt{VC}_L(f,x_0)}{\partial L}\Big|_{L^+,L=0}
&=\lim_{L\rightarrow0^+}\frac{\texttt{VC}_L(f,x_0)-\texttt{VC}_0(f,x_0)}{L} \leq\frac{2}{L}\max_{y\in[x_0-\frac{L}{2},x_0+\frac{L}{2}]}|f(x_0)-f(y)|.
\end{aligned}\]
Since $f$ is differentiable, it can be expanded as a Taylor series at the point $x_0$, i.e. $f(x)=f(x_0)+f'(x_0)(x-x_0)+o(x-x_0)$. Then we have
\[\begin{aligned}
\frac{\partial\texttt{VC}_L(f,x_0)}{\partial L}\Big|_{L^+,L=0} \leq\lim_{L\rightarrow0^+}\frac{2}{L}\max_{y\in[x_0-\frac{L}{2},x_0+\frac{L}{2}]}|f'(x_0)(y-x_0)+o(y-x_0)| \leq f'(x_0).
\end{aligned}\]
Meanwhile, 
\[\begin{aligned}
f'(x_0)&=\lim_{\Delta x\rightarrow0}\frac{f(x_0+\Delta x)-f(x_0-\Delta x)}{2\Delta x}\\
&\leq\lim_{\Delta x\rightarrow0}\frac{\texttt{VC}_{2\Delta x}(f,x_0)-\texttt{VC}_0(f,x_0)}{2\Delta x}=\frac{\partial\texttt{VC}_L(f,x_0)}{\partial L}\Big|_{L^+,L=0}.        
\end{aligned}\]
Combining the above two formulas, we have 
\(
f'(x_0)=\frac{\partial\texttt{VC}_L(f,x_0)}{\partial L}\Big|_{L^+,L=0}.
\)
\end{proof}
 
\begin{remark}
The \texttt{VC} of a function at a given point represents the maximum amplitude of the function within a fixed-radius neighborhood centered at that point. This value, which depends on a radius of $L/2$, captures the local characteristics of the function and may be related to the ease with which the function can be approximated by a neural network. For example, when $L$ is fixed, if a function exhibits large \texttt{VC} in certain regions and small \texttt{VC} in others, the regions with large \texttt{VC} may correspond to areas with rapid functional variation that are more difficult to approximate, whereas regions with small \texttt{VC} are relatively smoother and more amenable to approximation.
\end{remark}

{\begin{remark}
In some multi-scale problems, different dimensions may have different computational scales, so the definition of \texttt{VC} can be extended to the appropriate ${L}$ length for each dimension (${\bm{L}} = (L_1,L_2,\dots,L_n)$)
\begin{equation*}
\texttt{VC}_{\bm{L}}(f,{\bm x}) = \sup_{{\bm y}_1, {\bm y}_2 \in \Omega \cap \prod_{i=1}^{n} [{x}_i-\frac{L_i}{2}, {x}_i+\frac{L_i}{2}]} \ \vert f({\bm y}_1)-f({\bm y}_2)\vert. 
\end{equation*}
\end{remark}}

This definition captures the maximum change in the function \(f\) within the interval \(\Omega \cap \prod_{i=1}^{n} [x_i-L_i/2, x_i+L_i/2]\), providing a measure of the function's variability over this localized region of size(s) \(\bm{L}\). The \texttt{VC} is particularly useful in neural network approximation problems, where understanding how the objective function fluctuates in smaller intervals can inform the design of the network's architecture or its training strategy.

The \texttt{VC} reflects the variation of function values within an interval of length $\bm{L}$. When a function is smooth and differentiable, the {\ttfamily VC} value is directly proportional to the maximum derivative of the function over that interval. For instance, in the case of a fundamental one-dimensional linear function \( f(x) = x-1 \), the {\ttfamily VC} value is equal to $L$ for any given interval and $L$ value.

\begin{remark}
If  \(f\) is bounded on \(\prod_{i=1}^{n} [a_i, b_i]\), then \(\texttt{VC}_{\bm{L}}(f,\bm{x})\) is bounded.
\end{remark}

\begin{remark}
As a measure of a function's local complexity, the VC holds promise for a wide range of applications. In addition to its direct role in characterizing the rate of neural network approximation—illustrated in the example in Section \ref{motvation}—--it can also be utilized to enhance neural network algorithms. These enhancements include network preprocessing, the design of optimization algorithms, and the definition of loss functions. In Section \ref{sec5}, we introduce a novel norm derived from the VC. Based on this norm, we propose a new preprocessing method that leads to a significant reduction in both the initial $L^2$  error and the VC-based approximation error.
\end{remark}

\begin{remark}
The notion of VC used in this paper differs from the standard definition of the Vapnik–Chervonenkis (VC) dimension.
\end{remark}

While the \texttt{VC} provides a local measure of function variability, it is also essential to understand how these changes are distributed across the entire domain \(\prod_{i=1}^{n} [a_i, b_i]\). This leads to the concept of the \texttt{VC density}, which describes the distribution of different levels of change across the function.

{\color{black}\begin{definition}[\ttfamily VC Density]\label{LCDF}
A random variable $vc_L$ has density $F_{VC_L}(vc_L)$, where $F_{VC_L}(vc_L)$ is a non-negative Lebesgue-integrable function defined on $[0,+\infty)$, if
\(
{\rm Pr}[a\leq v_L\leq b]=\int_a^bF_{VC_L}(vc_L)dvc_L.
\) 
\end{definition}}

Fig.~\ref{VCD_eg} {shows} the \texttt{VC Density} functions of the piece-wise linear function (\ref{pw}). In the experiment, we take $L = 0.01$ and estimate the \texttt{VC Density} function using the Kernel Density Estimation (KDE) algorithm using the Gaussian kernel function\footnote{In the experiments below, this method will be used to estimate the \texttt{VC Density} function.}. 
Since $f_1$ and $f_2$ are the same when $x \leq 0$, their density functions have a peak at $\texttt{VC} = 0.04$. And $f_1$ is always equal to $0$ when $x > 0$, so the \texttt{VC Density} of $f_1$ has a peak at $\texttt{VC} = 0$. 

\begin{figure}[htb]
\centering
\includegraphics[width=0.9\textwidth,trim=40 0 40 0,clip]{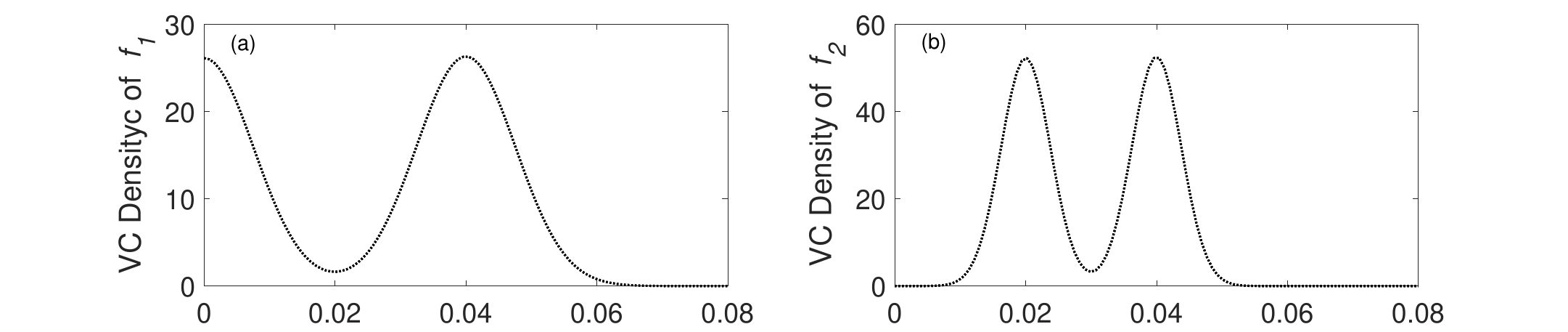} 
\caption{\color{black}{The \texttt{VC Density} function} of (a) the piece-wise function $f_1$, (b) the piece-wise function $f_2$ {in (\ref{pw})}.}
\label{VCD_eg}
\end{figure}

The {\ttfamily VC Density} function illustrates the distribution of {\ttfamily VC} values of functions. {\color{black}In the case of a smooth and differentiable function, the {\ttfamily VC} value reflects the distribution of function derivatives.} For example, considering a linear function \( f(x) = x-1 \), which satisfies $VC_L=L$, the {\ttfamily VC Density} function $F_{\texttt{VC}_{L}(f,{\bm x})}(\texttt{VC}_L=L)$ is 1, and it is 0 for all other {\ttfamily VC}  values.

The {\ttfamily VC Density} function $F_{\texttt{VC}_{L}(f,{\bm x})}(VC_L)$ gives the probability that the \texttt{VC} at a randomly chosen point \({\bm x}\) takes a specific {\texttt{VC} value. By studying this distribution, we can gain insights into the overall behavior of \(f\) over the interval, helping to identify regions where the function undergoes significant changes or remains relatively stable. Such knowledge can be applied to fine-tune neural network training, for example, by focusing more on regions where the function varies greatly.
As neural networks approximate the objective function, the \texttt{VC Density} of neural networks will also approximate the \texttt{VC Density} of the objective function. To facilitate the description of the approximation process of \texttt{VC Density}, we define the ratio between two {\ttfamily VC Density}.

\begin{definition}[\ttfamily VC Density Ratio]
The ratio between the {\ttfamily VC Density} of different functions \(f_1\) and \(f_2\) is defined as  
    \[\texttt{VCDR}_{{\bm x},L}(f_1,f_2)= 
\frac{{\mathcal P}_{\texttt{VC}_{L}(f_1,{\bm x})}}{{\mathcal P}_{\texttt{VC}_{L}(f_2,{\bm x})}}. 
\]
\end{definition}

If $f_1$ and $f_2$ are set to the neural networks function and the objective function, respectively, then the ratio of their \texttt{VC Densitys} is a more intuitive way to describe the approximation of \texttt{VC Densitys}. No matter the size of the \texttt{VC Density}, when the \texttt{VC Densitys} ratio at a point is closer to 1, the \texttt{VC Densitys} at that point can be considered to be better approximated. Of course, when the real \texttt{VC Density} approaches 0, the slight fluctuation of the \texttt{VC Density} of neural networks will lead to a drastic change in the \texttt{VC Density} ratio, and the \texttt{VC Density} ratio will have no reference significance in this case.

\subsection{Strong correlation between the loss function and the value change (\texttt{VC})}\label{sec3.2}

{In the motivation, we find that the approximation speed of the objective function is related to its slope. In different linear functions, the larger the slope, the slower the approximation. In the same function, the larger the slope, the slower the approximation. At the beginning of this section, we introduce the concept of \texttt{VC}, which can measure the local change of the function more comprehensively. Roughly speaking, {\ttfamily VC} can be used as an indicator to reflect the slope, so it is concluded that when the value of \(L\) is determined, the larger the {\ttfamily VC}, the slower the neural network approach. The impact of \texttt{VC} values under different $L$ values is not fully discussed. In this chapter, we will briefly explore the relationship between \(c\) under variable $L$ and neural network approximation.}

{In this section, we consider the same example as Ex. \ref{eg2}. In training, the neural networks are selected as a fully connected network with $5$ layers and $50$ elements per layer, and the optimization algorithm chooses the standard Adam algorithm with a learning rate of $0.01$.}
Fig.~\ref{VCloss}(a) and Fig.~\ref{VCloss}(b) show the trend of pointwise error in the process of neural networks approximating $f_0$ and $f_1$ {in (\ref{eg2})}, respectively. {Figs.~\ref{VCloss}(a) and (b)} show the total training results for the first $250$ steps. 

{In the objective functions $f_0$ and $f_1$, there are significant jump breaks when $x = 0$. Since the usual fully connected neural networks are continuous functions, there will always be large errors in neural networks near the jump points, which can be seen in the pointwise error graphs of Fig.~\ref{VCloss}(a) and Fig.~\ref{VCloss}(b). However, it can be seen in the \texttt{VC} change graphs of Fig.~\ref{VCloss}(c) and Fig.~\ref{VCloss}(d) that no matter how large $L$ is, the \texttt{VC} is relatively large in a small field of $x = 0$, which corresponds to the pointwise error Fig.~\ref{VCloss}(a) and Fig.~\ref{VCloss}(b). However, within the same objective function, it can also be seen that the larger  \texttt{VC} makes the function more difficult to approach. Fig.~\ref{VCloss}(a) shows that when approaching $f_0$, the error of neural networks at $x < 0$ is significantly smaller than the error at $x > 0$, which corresponds to the larger  \texttt{VC} on the left side in Fig.~\ref{VCloss}(c). Fig.\ref {VCloss}(b) also shows that the pointwise error on the left side of neural networks is significantly larger than that on the right side, which corresponds to the larger \texttt{VC} on the left side in Fig.~\ref{VCloss}(d).}

\begin{figure}[htbp]
\centering
\includegraphics[width = 1.0\textwidth, trim=100 0 80 0, clip]
{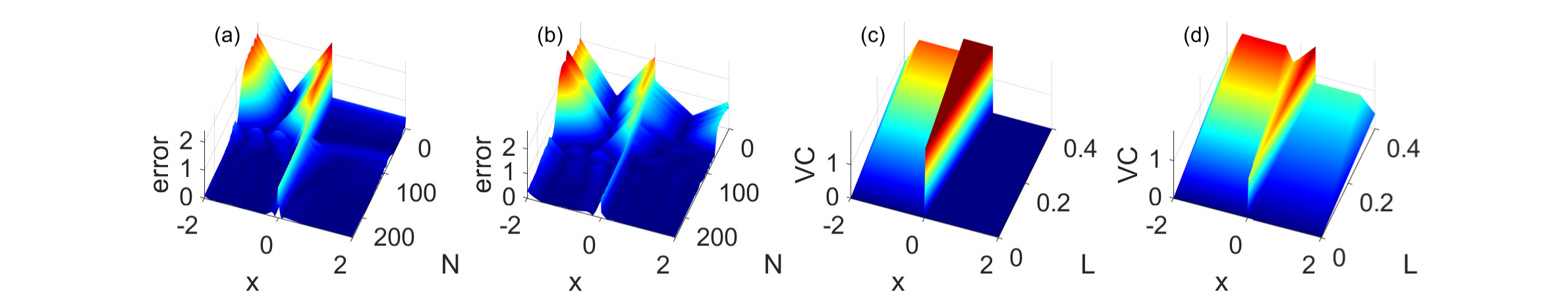} 
\caption{{\color{black} Correlation between point-wise errors during training and point-wise \texttt{VC} for different $L$.} (a) (b): pointwise error in Ex. \ref{eg2} (a) $f_0$, (b) $f_1$. (c) (d): pointwise \texttt{VC} in Ex. \ref{eg2}: (c) $f_0$, (d) $f_1$.}
\label{VCloss}
\end{figure}

\subsection{Theoretical properties of \texttt{VC}}

To better understand the behavior and characteristics of the \texttt{VC}, we now discuss several key theoretical properties. These properties are crucial in establishing the robustness and applicability of \texttt{VC} in various contexts, such as approximation, optimization, and analysis of neural networks. The \(\texttt{VC}_L(f) \) has the following properties. 

\begin{theorem}[Affine invariance with respect to \(f\)] 
Given a function \(f\) with \( \texttt{VC}_{L}(f, {{\bm x}})\) for \(\forall\ {{\bm x}}\), the \texttt{VC} of the function \(\kappa f\) on \({{\bm x}}\) is 
\(
\texttt{VC}_{L}(\kappa f+c, {{\bm x}})=\vert \kappa\vert \texttt{VC}_{L}(f, {\bm x}),\ \forall \ \kappa, c\in\mathbb{R}.
\) 
\end{theorem}

\begin{proof}  
We begin the proof with the definition   
\[
\texttt{VC}_{L}(f, {\bm x}) = \sup_{{\bm y}_1, {\bm y}_2 \in \Omega \cap \prod_{i=1}^{n} [{x}_i-\frac{L}{2}, {x}_i+\frac{L}{2}]} \vert f({\bm y}_1) - f({\bm y}_2) \vert.   
\]  
For the affine function \(\kappa f + c\), we get  
\[
\texttt{VC}_{L}(\kappa f + c, {\bm x}) = \sup_{{\bm y}_1, {\bm y}_2 \in \Omega \cap \prod_{i=1}^{n} [{x}_i-\frac{L}{2}, {x}_i+\frac{L}{2}]} \left\vert \kappa f({\bm y}_1) + c - (\kappa f({\bm y}_2) + c) \right\vert. 
\]  
The constant term \(c\) cancels out, leading to  
\[
\texttt{VC}_{L}(\kappa f + c, {\bm x}) = \sup_{{\bm y}_1, {\bm y}_2 \in \Omega \cap \prod_{i=1}^{n} [{x}_i-\frac{L}{2}, {x}_i+\frac{L}{2}]} \vert \kappa \vert \cdot \vert f({\bm y}_1) - f({\bm y}_2) \vert.  
\]  
Since \(\vert \kappa \vert\) is constant, we have  
\[
\texttt{VC}_{L}(\kappa f + c, {\bm x}) = \vert \kappa \vert \cdot \texttt{VC}_{L}(f, {\bm x}).  
\]  
Thus,   
\(
\texttt{VC}_{L}(\kappa f + c, {\bm x}) = \vert \kappa \vert \texttt{VC}_{L}(f, {\bm x}), \quad \forall \, \kappa, c \in \mathbb{R},  
\)   
and the proof is complete. 
\end{proof}

This property ensures that the \texttt{VC} is invariant to affine transformations, making it a reliable measure when dealing with re-scaled or shifted functions.

% {\bf Theorem 2.}   (
\begin{theorem}[Uniqueness of \texttt{VC}] The value of \texttt{VC} of the function \(f\) at \({\bm x}\) is unique. 

\end{theorem}

\begin{proof}
We start with the definition of the \texttt{VC} of \(f\) at \({\bm x}\) over the corresponding interval.  
\[
\texttt{VC}_{L}(f, {\bm x}) = \sup_{{\bm y}_1, {\bm y}_2 \in \Omega \cap 
 \prod_{i=1}^{n} [{x}_i-\frac{L}{2}, {x}_i+\frac{L}{2}]} \vert f({\bm y}_1) - f({\bm y}_2) \vert.
\]
Assume for contradiction that two distinct values, \(\texttt{VC}_{L,1}(f, {\bm x})\) and \(\texttt{VC}_{L,2}(f, {\bm x})\), there exists 
\( 
\texttt{VC}_{L,1}(f, {\bm x}) \neq \texttt{VC}_{L,2}(f, {\bm x}).
\) 
Since both are supremums over the same interval, this contradicts the uniqueness of the supremum. 
This completes the proof.
\end{proof}

This ensures that the measure is well-defined and reliable for approximating the behavior of functions, without ambiguity in its calculation.

\begin{theorem}[Symmetric in-variance with respect to \({\bm x}\)] 
If a function \(f(\bm x)\) is symmetric with respect to \({\bm x}={\bm d}\in\mathbb{R}^n\), then 
\begin{equation}
\texttt{VC}_{L}(f, {\bm x})=\texttt{VC}_{L}(f, 2{\bm d} -{\bm x}).
\end{equation}
\end{theorem}

\begin{proof}
We start by assuming that the function \(f(\bm{x})\) is symmetric about \({\bm x} = {\bm d}\). This means that for any \({\bm y} \in \mathbb{R}^n\), the following holds:
\(
f({\bm y}) = f(2{\bm d}  - {\bm y}).
\) 
Next, we consider the \texttt{VC} at a point \({\bm x}\), which is given by 
\[
\texttt{VC}_{L}(f, {\bm x}) = \sup_{{\bm y}_1, {\bm y}_2 \in \Omega \cap \prod_{i=1}^{n} [{x}_i-\frac{L}{2}, {x}_i+\frac{L}{2}]} \vert f({\bm y}_1) - f({\bm y}_2) \vert.
\]
Since the function is symmetric about \({\bm x} = {\bm d}\), we can substitute \({\bm y}_1\) and \({\bm y}_2\) using their symmetric counterparts \(2{\bm d}  - {\bm y}_1\) and \(2{\bm d}  - {\bm y}_2\). Thus, for \({\bm x}' = 2{\bm d}  - {\bm x}\), we have 
\[
\texttt{VC}_{L}(f, {\bm x}') = \sup_{{\bm y}'_1, {\bm y}'_2 \in \Omega \cap  \prod_{i=1}^{n} [{x}'_i-\frac{L}{2}, {x}'_i+\frac{L}{2}]} \vert f({\bm y}_1') - f({\bm y}_2') \vert.
\]
By the symmetry property of \(f\), we know that 
\(
f({\bm y}_1') = f(2{\bm d}  - {\bm y}_1')\)  and \(f({\bm y}_2') = f(2{\bm d}  - {\bm y}_2')\). 
Hence, the supremum over the interval \([2{\bm d}  - {\bm x} - L/2, 2{\bm d}  - {\bm x} + L/2]\) will result in the same value as the supremum over the corresponding region since the function behaves identically in both intervals due to the symmetry. Thus, we conclude
\(
\texttt{VC}_{L}(f, {\bm x}) = \texttt{VC}_{L}(f, 2{\bm d}  - {\bm x}).
\) 
This proves the theorem.
\end{proof}

This symmetry property is particularly useful when analyzing functions that exhibit symmetric behavior, as it allows us to reduce the complexity of calculations by exploiting the inherent symmetry of the function.

\section{VC-related tendencies in neural network approximation}\label{sec4}

{\color{black} In this section, we discuss two neural network approximation tendencies related to \texttt{VC}: the \texttt{VC} tendency and the minority tendency. These tendencies respectively characterize the direct relationship between neural network approximation and \texttt{VC} values, as well as the evolving distribution of \texttt{VC} during the training process. We begin by presenting several fundamental yet illustrative experiments to validate these relationships and trends. Additional examples will be made available on our website \footnote{\href{https://github.com/pxie98/L-Change/tree/main/VC}{\ttfamily https://github.com/pxie98/L-Change/tree/main/VC}} and explored further in future work.}

\subsection{\texttt{VC}-tendency phenomenon}

The first key phenomenon is the \texttt{VC}-tendency phenomenon, which illustrates that the approximation process of neural networks is highly correlated with the \texttt{VC} value of the objective function. We employ two experiments to demonstrate the \texttt{VC}-tendency phenomenon: the approximation of a monochrome image and the simulation of flow past a cylinder.

\subsubsection{Approximation to a monochrome picture}\label{mono}

\begin{figure}[htb]
\centering
\includegraphics[width=0.9\textwidth,trim=0 0 0 0,clip]{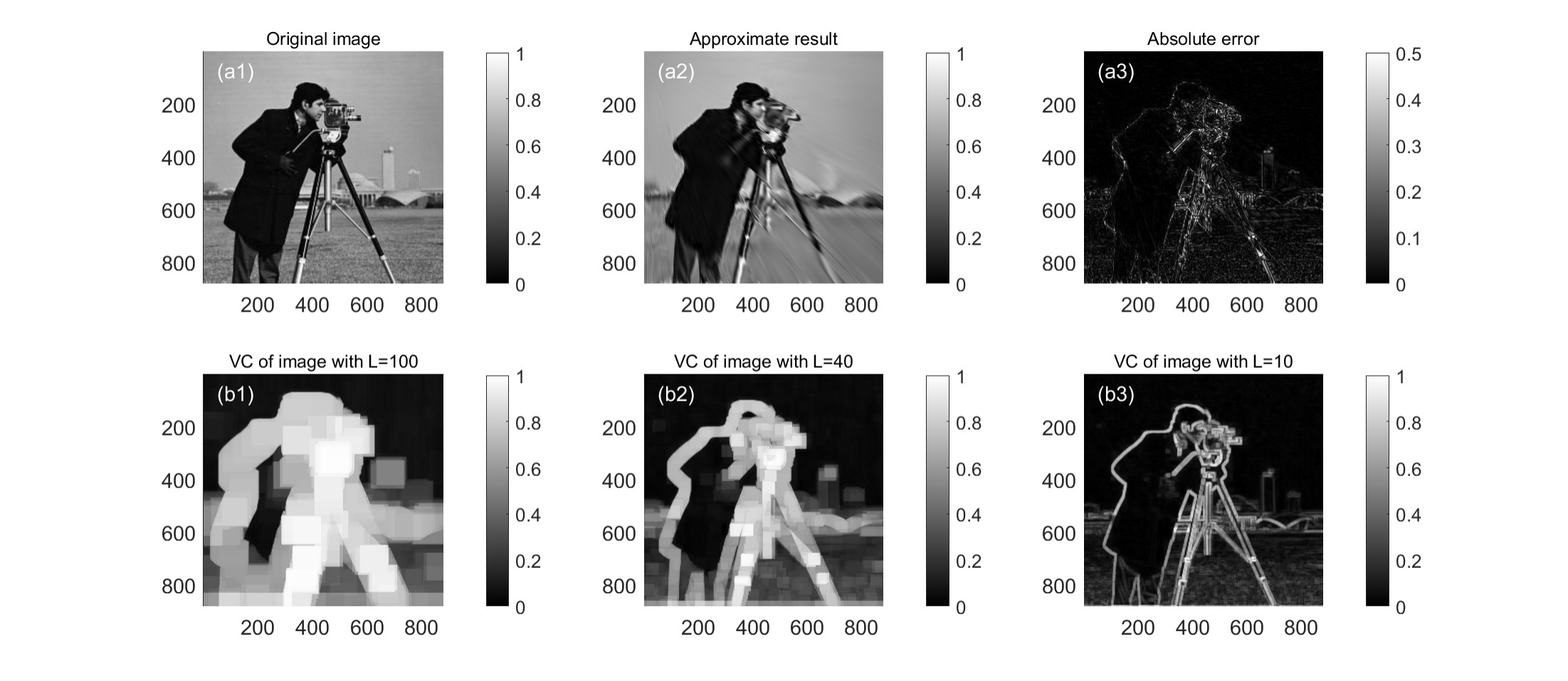}   
\caption{{\color{black} Direct approximation of a monochrome image and \texttt{VC} with different diameters}.
(a1) Original monochrome image. (a2) Neural network approximation result.
(a3) Absolute error between the approximation and the original. (b1)–(b3) \texttt{VC} of the original image computed with different neighborhood diameters $L$: (b1) $L=100$ pixels, (b2) $L=40$ pixels, (b3) $L=10$ pixels.}
\label{fig:picture}
\end{figure}

\begin{figure}[htb]
\centering
\includegraphics[width=1.0\textwidth,trim=0 0 0 0,clip]{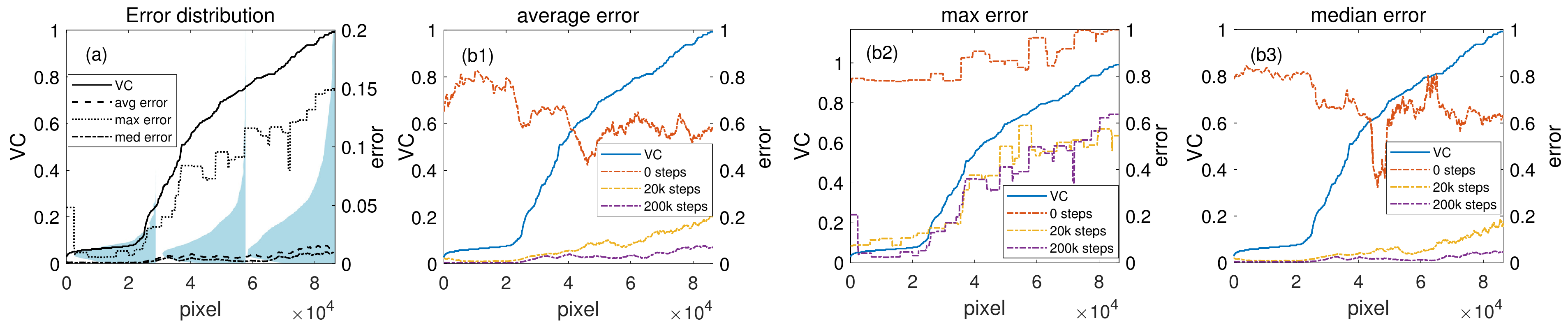}   
\caption{{\color{black} \texttt{VC}-tendency phenomenon in monochrome image approximation.} (a) Point-by-point error curve sorted by \texttt{VC} value, and point error size distribution of different \texttt{VC} distribution pixels. (b1)-(b3) Correlation of error curve with \texttt{VC} for three different smoothing methods.}
\label{fig:VCvsLoss}
\end{figure}

The first experiment involves the approximation of a monochrome image. Here, the neural network takes two inputs corresponding to the spatial coordinates of each pixel. The output is the grayscale intensity, ranging from $[0,1]$, where an output of 1 represents pure black and $0$ represents pure white. A fully connected neural network with five layers, each containing 300 neurons, is trained using the Adam optimizer for 200,000 steps. The approximation results and corresponding error comparison are in Fig.~\ref{fig:picture} (a),(b) and (c). Notably, the white regions in Fig.~\ref{fig:picture}(c), which indicate pixels with high approximation error at the end of training, resemble the contours of the original image. Fig.~\ref{fig:picture}(b1), (b2), and (b3) depict the \texttt{VC} distribution of the objective function under different values of $L$. As $L$ approaches zero, the \texttt{VC} distribution increasingly resembles the contours of the original function. This implies that these high-error regions correspond to areas of high \texttt{VC} value in the objective image.

To illustrate the correlation between \texttt{VC} and prediction error, we first sort all $880 \times 880$ data points in the graph based on the \texttt{VC} value computed within a local field of radius 50 pixels ($L=101$). The corresponding error (after training) values are then plotted in the same order as the sorted \texttt{VC} values. Fig.~\ref{fig:VCvsLoss}(a) presents the error curves under three different smoothing methods. Here, the average, maximum, and median errors are shown (denoted as Avg error, Max error, and Med error, respectively), using a smoothing radius of 2,000 pixels. Additionally, the three flipped blue areas indicate the error distributions when the data points are divided into three groups according to their \texttt{VC} magnitudes. Both the overall trend and the distribution plots suggest that data points with larger \texttt{VC} values tend to exhibit higher prediction errors.

Fig.~\ref{fig:VCvsLoss}(b1)–(b3) further illustrates the evolution of error with training. These subfigures show the sorted errors (by \texttt{VC} value) at three stages: the initial state, after 20,000 training steps, and after 200,000 steps, using the same three smoothing strategies. While the initial error distribution shows no clear dependency on \texttt{VC}, it becomes evident over time that regions with higher \texttt{VC} values consistently retain larger errors as training progresses.

\begin{sloppypar}
{\color{black} The example of approximating monochrome images demonstrates that when using a two-dimensional function to approximate data, the \texttt{VC} value of the objective function is positively correlated with the final approximation error. In other words, there appears to be a potential ``\texttt{VC} tendency principle'', which we will refer to as the VC tendency phenomenon. This phenomenon may exist in many application domains. In the next subsection, we further investigate this phenomenon using the example of 2-dimensional flow past a cylinder.}
\end{sloppypar}

\subsubsection{Approximation to the flow past a cylinder}\label{cylinder}

In this section, we consider the approximation of the flow around a cylinder in two-dimensional incompressible fluid using neural networks. 
{\color{black}The flow past a cylinder is a canonical problem in fluid dynamics, widely studied due to its fundamental significance and broad applicability in engineering and environmental contexts.} 
Understanding flow over bluff bodies like circular cylinders is important for designing structures affected by fluid forces, such as bridge piers, offshore platforms, heat exchangers, and aircraft parts. The flow is complex due to phenomena like boundary layer separation, vortex shedding, transition to turbulence, and wake formation, making it both theoretically challenging and practically important.

Numerical simulation has become an indispensable tool in investigating flow past a cylinder, offering insights into flow structures, force coefficients, and transition mechanisms that are difficult to obtain through experiments alone. High-fidelity approaches such as Direct Numerical Simulation (DNS) \cite{moin1998direct}, Large Eddy Simulation (LES) \cite{mason1994large}, and Reynolds-Averaged Navier-Stokes (RANS) methods \cite{alfonsi2009reynolds} enable researchers to analyze the influence of Reynolds number, surface roughness, and free-stream turbulence on flow behavior. 
In recent years, deep learning has emerged as a powerful tool in the field of computational fluid dynamics (CFD), offering new opportunities to accelerate simulations, enhance prediction accuracy, and reduce computational costs. Traditional CFD methods, based on solving the Navier–Stokes equations numerically, often require significant time and computational resources, especially for complex or high-fidelity simulations. Deep learning techniques, particularly convolutional neural networks (CNNs) \cite{guo2016convolutional}, recurrent neural networks (RNNs), and physics-informed neural networks (PINNs) \cite{wang2025simulating}, have been increasingly adopted to overcome these limitations. These models can learn complex flow patterns from data, approximate solutions to governing equations, or serve as efficient surrogates for conventional solvers. Applications of deep learning in CFD now span a wide range of areas, including turbulence modeling, aerodynamic design optimization, real-time flow prediction, and uncertainty quantification.

Typically, the motion of a two-dimensional incompressible fluid can be described by the incompressible Navier-Stokes (NS) equations 
\begin{equation}\label{NS}
\left\{
\begin{aligned}
&\frac{\partial \mathbf{u}}{\partial t} + (\mathbf{u}\cdot\nabla)\mathbf{u} = -\nabla p + \nu \Delta \mathbf{u},\\
&\nabla\cdot\mathbf{u} = 0, 
\end{aligned}
\right.
\end{equation}
where ${\bf u}$ is the fluid velocity, $\nabla p$ is the pressure gradient and $\nu$ is the kinematic viscosity.

In the numerical experiments, a neural network is employed to approximate the evolution of flow past a cylinder within a square domain. The flow data are obtained from direct numerical simulations as reported in \cite{trebotich2015adaptive}. In the setup, the domain has a length of $l = 16$ and a width of $w = 8$, with a cylinder located at $x = 1$ and a radius of $0.125$. The kinematic viscosity is set to $\nu = 0.0004167$, corresponding to a Reynolds number of $300$.

Similar to the analysis conducted for monochrome image approximation, we examine the relationship between the \texttt{VC} of the objective function and the approximation error. Our focus is on the correlation between the 3-dimensional \texttt{VC} and the error, where the 3-dimensional \texttt{VC} quantifies local variations around each point in the XYT space. Fig.~\ref{fig:VCvsLoss_cyl_3d}(a) reveals a strong correlation between the 3-dimensional \texttt{VC} and the approximation error, suggesting that regions with higher \texttt{VC} values tend to exhibit greater errors and are more challenging for the neural network to approximate accurately.
Furthermore, by sorting all residual points according to their \texttt{VC} values, we observe distinct error distributions across different \texttt{VC} ranges. Points with larger \texttt{VC} values are more likely to have higher errors; however, the majority of residuals—regardless of \texttt{VC}—remain within the low-error range. Figs.~\ref{fig:VCvsLoss_cyl_3d}(b), (c), and (d) respectively depict the relationship between three types of local error and the \texttt{VC} of the objective function at different training stages, providing numerical evidence of the consistently strong correlation between \texttt{VC} and error throughout the training process.

\begin{figure}
    \centering
    \includegraphics[width=1.0\linewidth]{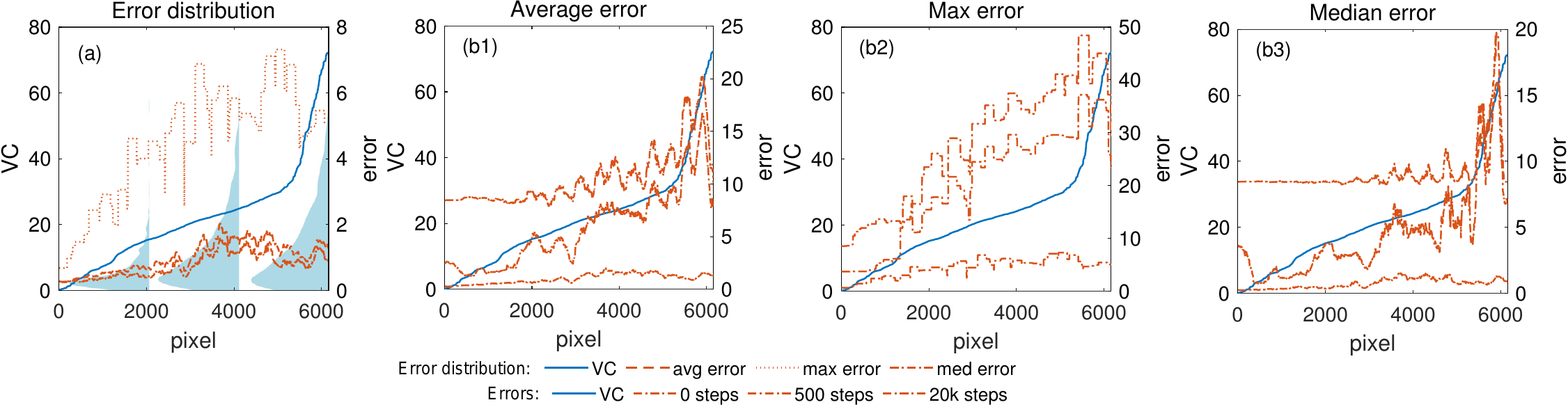}
    \caption{{\color{black} 3-dimensional \texttt{VC}-tendency phenomenon in the flow past a cylinder.} (a): the trends of three types of errors (average, maximum, and median) for pixel points sorted according to 3-dimensional  \texttt{VC} ordering in the flow past a cylinder, and the flipped blue area represent the error distribution of three intervals divided by \texttt{VC}. (b1)-(b3): changes in three types of errors throughout training after the pixel points are sorted by 3-dimensional  \texttt{VC} ordering.}
    \label{fig:VCvsLoss_cyl_3d}
\end{figure}

The above observations indicate that when considering the 3-dimensional  \texttt{VC}, the \texttt{VC} value of the objective function remains positively correlated with the final approximation error, further validating the VC tendency phenomenon. Results for higher-dimensional cases will be presented in future experiments. 
% Dimensionality reduction is a crucial component in deep learning, with methods such as principal component analysis (PCA) and manifold learning being widely used. 
In the following, we investigate whether neural networks trained on high-dimensional systems still exhibit the aforementioned properties when observed from a lower-dimensional perspective.
Therefore, we are particularly interested in examining the relationship between \texttt{VC} and approximation error within the spatial domain (XY), and in assessing how this relationship changes as time progresses.We refer to the (XY)-domain \texttt{VC} as the reduced-order \texttt{VC}, as it captures variations only within the spatial domain.

\begin{figure}[htb]
\centering
\includegraphics[width=1.0\textwidth,trim=110 0 110 0,clip]{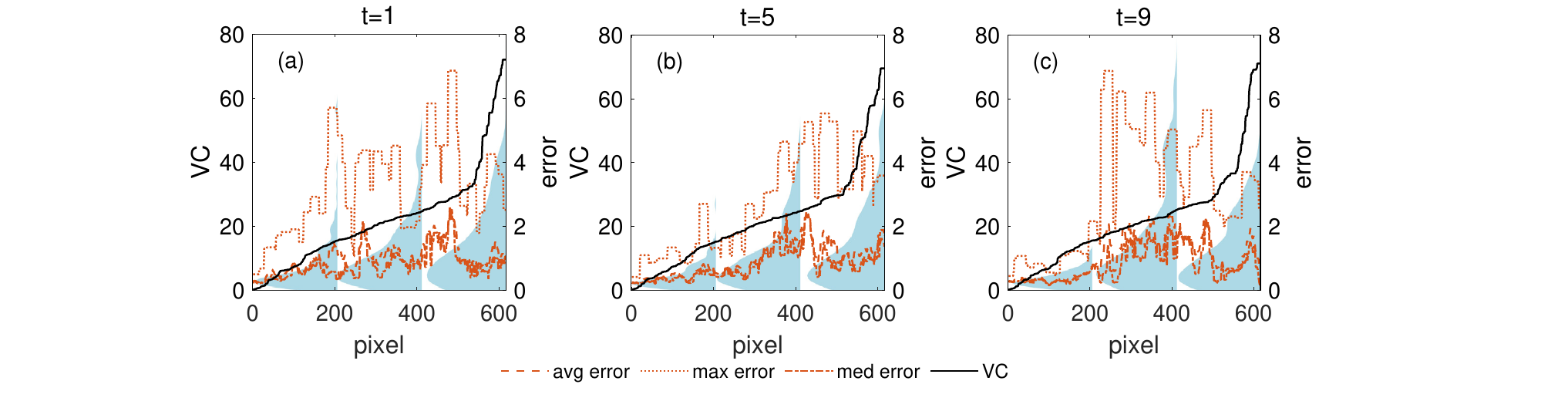}    
\caption{Error distribution (after training) sorted by reduced order \texttt{VC} in 2-dimensional flow past a cylinder.
(a)(b)(c) respectively present the correlation between three types of errors observed after training and the reduced-order \texttt{VC} at three different time ($t=1,5,9$) instances in approximation of flow past a cylinder. (All pixels are ranked according to their corresponding reduced order \texttt{VC}, with the reduced order \texttt{VC} calculation radius set to 2 pixels).}
\label{fig:VCvsLoss_cyl_2d1}
\end{figure}

{\color{black} Figs~\ref{fig:VCvsLoss_cyl_2d1} illustrates the relationship between the reduced-order 2-dimensional \texttt{VC} and the pointwise error in 2-dimensional flow past a cylinder. It can be observed that, regardless of the time instance, once the data points are sorted by the reduced-order \texttt{VC}, all three types of smoothed error consistently exhibit a positive correlation with the \texttt{VC}.}

\begin{figure}[htb]
\centering
\includegraphics[width=0.9\textwidth,trim=0 0 0 0,clip]{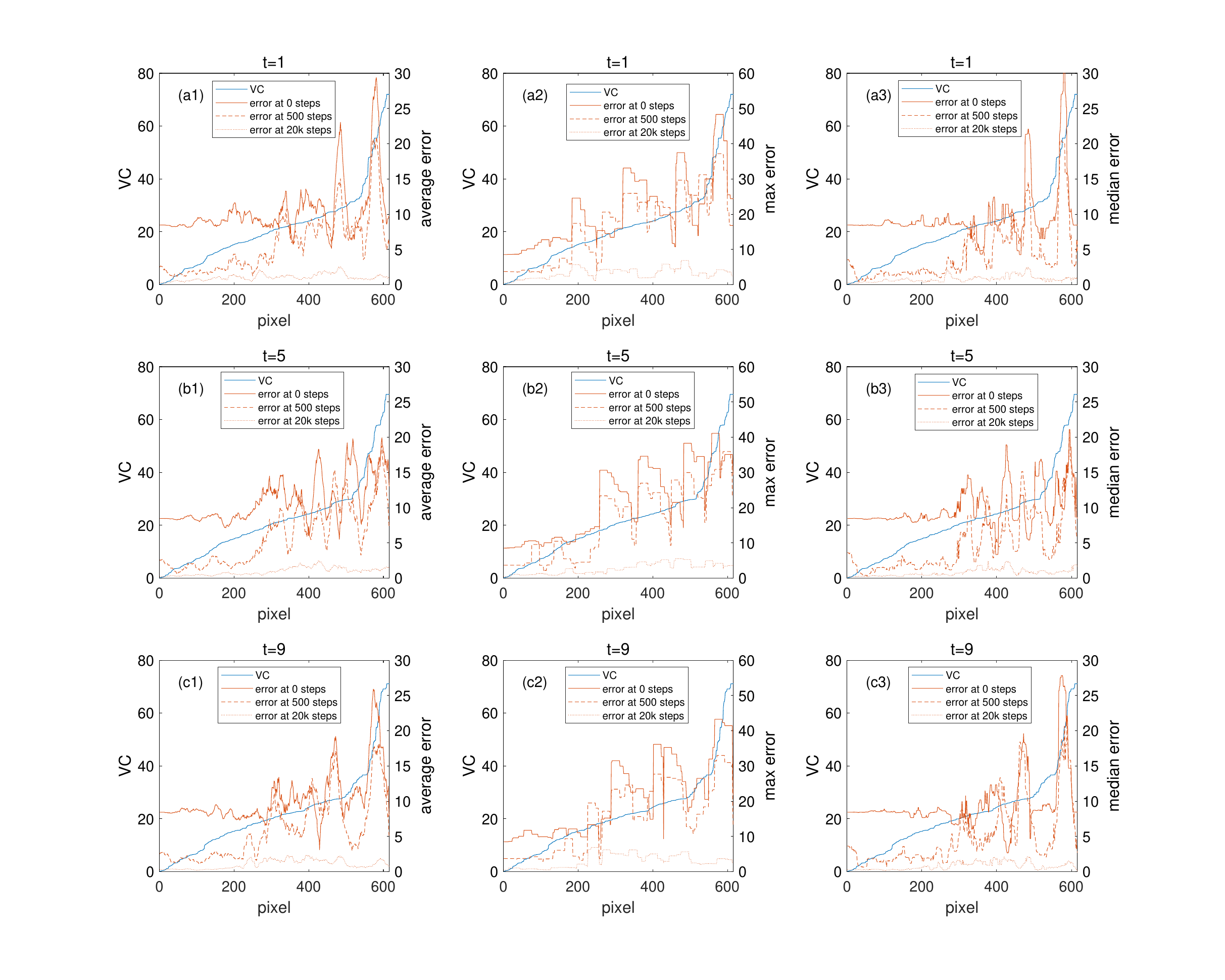}   
\caption{{\color{black} Reduced order 2-dimensional \texttt{VC}-tendency phenomenon in the 2-dimensional flow past a cylinder.} Figs. (a1)-(a3), (b1)-(b3), (c1)-(c3) respectively illustrate the correlation between three types of errors and reduced order \texttt{VC} at different training steps during flow past a cylinder approximation, captured at three time instants ($t=1, 5, 9$).}
\label{fig:VCvsLoss_cyl_2d2}
\end{figure}

\subsection{Minority-tendency phenomenon}

In the previous section, we investigated the relationship between \texttt{VC} and neural networks in the context of approximating images and physical phenomena. The results indicate that regions with smaller \texttt{VC} values exhibit faster approximation rates. \color{black}{For instance, in a previous low-Reynolds-index cylinder wake flow approximation experiment, the majority of pixel-wise \texttt{VC} was relatively small (at least $3/4$ of the pixels had \texttt{VC} less than $\max(\texttt{VC})/2$). The results indicated that regions with smaller \texttt{VC} values could be approximated more rapidly. However, in practical applications, regions with large and small \texttt{VC} values are often treated as a single integrated domain. Under such circumstances, relying solely on error curves sorted by \texttt{VC} does not reveal which subset of the domain can be approximated more efficiently. }

In this section, we study how the \texttt{VC} distribution evolves during the approximation process from a distributional perspective. Furthermore, we introduce the concept of \texttt{VC} Density (Definition \ref{LCDF}) to characterize the distribution of \texttt{VC}.

\begin{figure}
    \centering
    \includegraphics[width=0.85\linewidth]{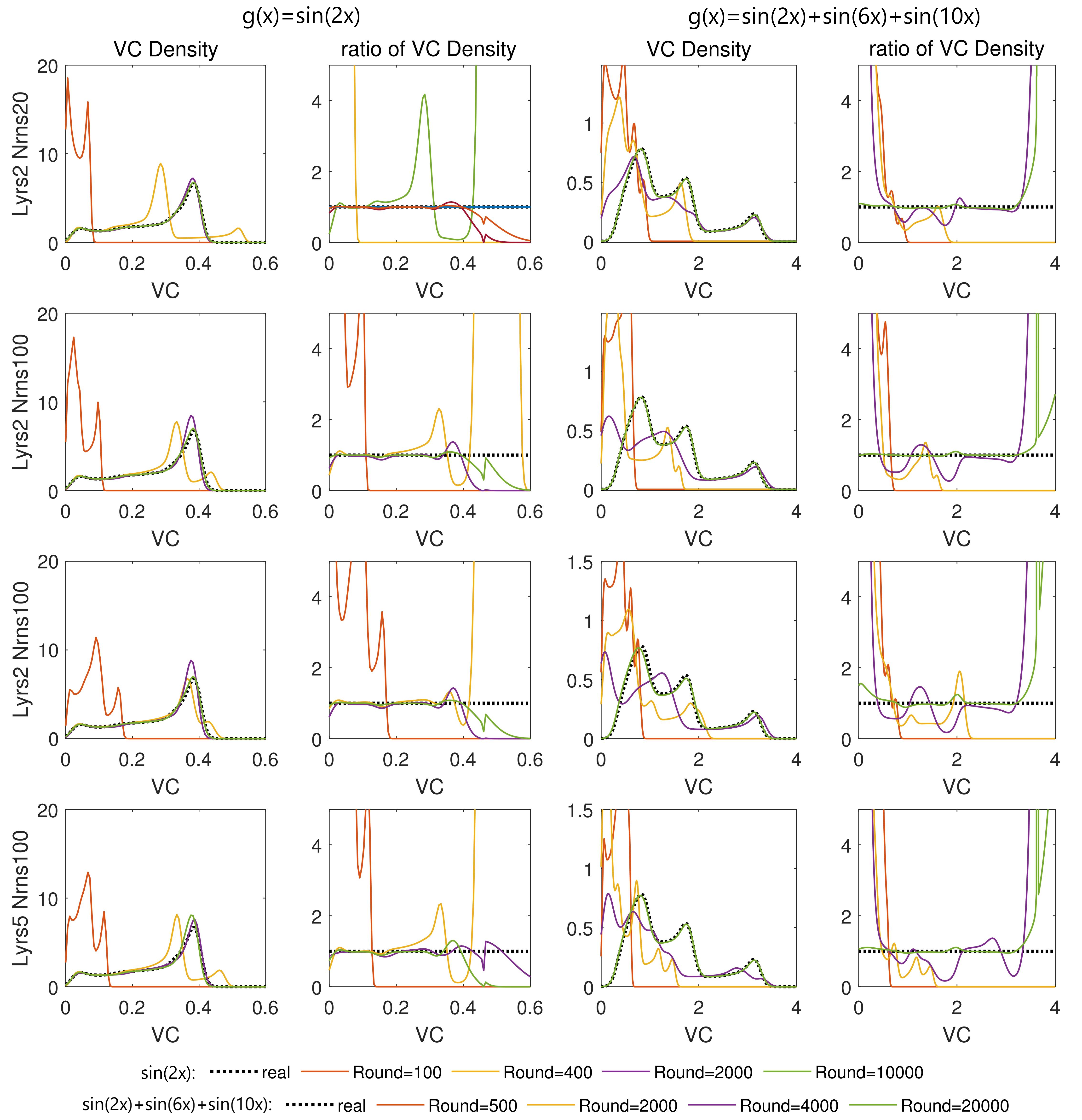}
    \caption{The \texttt{VC Density} evolution of \texttt{VC} when the neural networks approximate  $\sin(2x)$ (left two columns) and  $\sin(2x)+\sin(6x)+\sin(10x)$ (right two columns) by the algorithm \texttt{Adam} with a learning rate of $10^{-2}$.}
    \label{VC Density2}
\end{figure}

A natural query is where the \texttt{VC} of the function being approximated will be approximated sooner. In this section, we examine the evolution of \texttt{VC Density} during the approximation of neural networks with a fixed radius ($L=0.2$). To demonstrate the universality of the phenomenon, three representative optimization algorithms are chosen for the experiment, namely \texttt{Adam} (with initial learning rates of $10^{-2}$ and $10^{-3}$), \texttt{SGD} (with an initial learning rate of $10^{-2}$), and \texttt{LBFGS}.
The numerical results of \texttt{Adam} is shown in the manuscript and the related animations are posted online at an online repository\footnote{\href{https://pxie98.github.io/Minority-first.html}{https://pxie98.github.io/Minority-first.html}, also in \href{https://github.com/pxie98/L-Change/tree/main/VC}{\ttfamily https://github.com/pxie98/L-Change/tree/main/VC}.}.  Considering the representativeness of \texttt{VC Density}, the objective functions are determined as $\sin(2x)$  and $\sin(2x)+\sin(6x)+\sin(10x)$. 

{\color{black}The 1st and 3rd columns of Fig.~\ref{VC Density2}  show the \texttt{VC Density} evolution results when the neural {network} approximates $\sin(2x)$ and $\sin(2x)+\sin(6x)+\sin(10x)$ via the algorithm \texttt{Adam}  with the learning rate as $10^{-2}$. {\color{black}The 2nd and 4th columns of Figs.~\ref{VC Density2} show the ratio of \texttt{VC Density}, \(\texttt{VCDR}_{{\bm x},L}(\psi_{\rm NN},f)\). It is easy to see that the \texttt{VC} scale curve rises on the left and right sides. The reason why the ratio rises on the left side is that in the initial state, the neural networks function approaches $0$, so the density function approaches $0$, and the value is larger, which leads to a larger scale. The larger proportion on the right column is because the \texttt{VC} density of the 
objective function tends to $0$ on the right, so the proportion is larger in the case of non-zero neural networks

The convergence of the neural network to the \texttt{VC density} of the objective function is equivalent to the \texttt{VC Density Ratio}'s convergence to 1. It can be seen in the left two columns of {\color{black}Fig.~\ref{VC Density2}} that the \texttt{VC Density} referring to a small value of \texttt{VC} is almost already well approximated at ${\rm Round}=400$. However, even when ${\rm Round} >2,000$, \texttt{VC Density} referring to a large value of \texttt{VC} still exhibits a large error (also true for the ratio curve in the right-hand side). This phenomenon also exists in the approximation of $\sin(2x)+\sin(6x)+\sin(10x)$. Most of the examples in Fig.~\ref{VC Density2} show that when $\texttt{VC}>2.1$ and ${\rm Round}=4000$, \texttt{VC Density} has been basically learned. However, for $\texttt{VC}<2.1$, ${\rm Round}=4000$ is far from enough for \texttt{VC Density} to approximate the \texttt{VC Density} of the objective function.

{\color{black}Table~\ref{tab1} and Table~\ref{tab2}} exhibit the \texttt{VC Density Ratio} of neural network approximation for $\sin(2x)$ and $\sin(2x)+\sin(6x)+\sin(10x)$, respectively. Table~\ref{tab1} shows that in four different neural network parameter settings, the \texttt{VC Density} at \texttt{VC}$=0.08$ and \texttt{VC}$=0.18$ is always significantly less than the \texttt{VC Density} at \texttt{VC}$=0.28$ and \texttt{VC}$=0.38$ when Round$=400$.  
The results demonstrate that, in this case, regions with smaller \texttt{VC} values exhibit earlier convergence in \texttt{VC density}. 
Table~\ref{tab2} shows different results. {The} \texttt{VC Density} at \texttt{VC}$=0.8$ and \texttt{VC}$=1.6$ is always significantly more than the \texttt{VC Density} at \texttt{VC}$=2.4$ and \texttt{VC}$=3.2$ when Round$=4,000$. In this example, \texttt{VC Density} first converges when the variable \texttt{VC} is large.  
\color{black}{In this case, regions with larger \texttt{VC} exhibit faster convergence in \texttt{VC density}. These two seemingly contradictory examples suggest a potential underlying principle governing neural network convergence: during approximation, neural networks may converge more rapidly at regions with lower \texttt{VC density}—a phenomenon we tentatively term the {\it minority-tendency phenomenon}. 
}

\begin{footnotesize}
\begin{table}[htbp]
\centering
\setlength{\tabcolsep}{3pt}  
\begin{minipage}[t]{0.47\textwidth}
\caption{{\ttfamily VC density ratio} for $\sin(2x)$ (trained with \texttt{Adam}, lr = $10^{-2}$)\label{tab1}} \footnotesize
\begin{tabular}{@{}lcccc@{}} 
\toprule
VC & 0.08 & 0.18 & 0.28  & 0.38  \\
VCDR & 1.35 & 1.71 & 2.09 & 6.60 \\
\midrule
\textbf{L: 2 N: 20} &&&&\\
Round=100   & 2.02  & 0 & 0 & 0 \\
Round=400   & \underline{0.97} & 1.14 & 3.98 & 0.079 \\
Round=2,000  & 0.99 & 0.98 & \underline{1.00} & 1.10 \\
Round=10,000 & \underline{1.00} & 0.99 & 0.99 & 1.02 \\
\midrule
\textbf{L: 2, N: 100} &&&&\\
Round=100   & 3.72 & 0 & 0 & 0 \\
Round=400   & \underline{0.98} & 1.03 & 1.26 & 0.152 \\
Round=2,000  & 0.97 & 0.93 & \underline{1.02} & 1.28 \\
Round=10,000 & \underline{1.00} & 0.99 & 0.99 & 1.06 \\
\midrule
\textbf{L: 5, N: 20} &&&&\\
Round=100   & 6.51 & 0.082 & 0 & 0 \\
Round=400   & \underline{1.02} & 0.93 & 1.10 & 0.78 \\
Round=2,000  & 0.97 & 0.92 & \underline{1.01} & 1.33 \\
Round=10,000 & \underline{1.00} & 0.98 & 0.99 & 1.06 \\
\midrule
\textbf{L: 5, N: 100} &&&&\\
Round=100   & 6.11 & 0 & 0 & 0 \\
Round=400   & \underline{0.99} & 1.02 & 1.31 & 0.117 \\
Round=2,000  & \underline{0.98} & 0.86 & 1.06 & 1.10 \\
Round=10,000 & 0.98 & 0.94 & \underline{0.99} & 1.23 \\
\bottomrule
\end{tabular}
\end{minipage}
\hfill
\begin{minipage}[t]{0.52\textwidth}
\caption{{\ttfamily VC density ratio} for $\sin(2x)+\sin(6x)+\sin(10x)$ (trained with \texttt{Adam}, lr = $10^{-2}$)\label{tab2}} \footnotesize 
\begin{tabular}{@{}lcccc@{}} 
\toprule
VC & 0.8 & 1.6 & 2.4 & 3.2  \\
VCDR & 0.78 & 0.47 & 0.09 & 0.20 \\
\midrule
\textbf{L: 2, N: 20} &&&&\\
Round=500   & 0.53 & 0 & 0 & 0 \\
Round=2,000  & 0.72 & 1.00 & 0 & 0 \\
Round=4,000  & 0.80 & 0.70 & \underline{0.96} & 0.91 \\
Round=20,000 & \underline{1.00} & 0.97 & 0.98 & 1.03 \\
\midrule
\textbf{L: 2, N: 100} &&&&\\
Round=500   & $5.7e{-8}$ & 0 & 0 & 0 \\
Round=2,000  & 0.32 & 0.43 & 0 & 0 \\
Round=4,000  & 0.52 & 0.71 & \underline{0.93} & 0.91 \\
Round=20,000 & 1.00 & 0.97 & \underline{1.00} & 1.00 \\
\midrule
\textbf{L: 5, N: 20} &&&&\\
Round=500   & 0.86 & 0 & 0 & 0 \\
Round=2,000  & 0.49 & 0.42 & $7.6e{-7}$ & 0 \\
Round=4,000  & 0.56 & 0.45 & 0.91 & \underline{0.95} \\
Round=20000 & 0.99 & 0.96 & \underline{0.99} & 0.99 \\
\midrule
\textbf{L: 5, N: 100} &&&&\\
Round=500   & 0 & 0 & 0 & 0 \\
Round=2,000  & 0.79 & 0.02 & 0 & 0 \\
Round=4,000  & 0.73 & 0.31 & \underline{1.12} & 0.33 \\
Round=20,000 & 1.00 & 0.97 & \underline{0.99} & 1.01 \\
\bottomrule
\end{tabular}
\end{minipage}
\end{table}
\end{footnotesize}

We can observe that lower-proportion \texttt{VC}s are given higher tendency when using neural networks for approximation. This result is not immediately {obtained} and differs from the intuitive understanding of the approximation process in traditional numerical computation. We refer to this phenomenon as the minority-tendency phenomenon in neural network approximation. {This discovery gives us a deeper understanding of how neural networks approach objective functions and may help us design more efficient deep learning algorithms.}

\section{An improved pre-processed neural network approximation algorithm}\label{sec5}

In this section, we first design a heuristic experiment, and based on its results, we propose a novel metric grounded in \texttt{VC} theory to quantify the variation between the objective function and the neural network within a local neighborhood. Building on this metric, we further introduce a novel preprocessing method for neural networks.

\subsection{Observation}

To inform the design of the pre-processing algorithm, we first adopt a two-stage training approach to observe the impact of the initial training (pre-training) on subsequent training. In the numerical experiments, an intermediate objective function $g$ is introduced. The neural network is first trained to approximate the intermediate objective function within a fixed number of steps, followed by training to approximate the final objective function $f$ for another fixed number of steps. %The experiments are conducted under the following four cases. 

\begin{example}[pre-training on one-dimensional linear function approximation]\label{eg3} For objective $f(x)=10x$ on a interval $[-1,1]$, the four strategies are adopted:
A: $g(x)=-100x$, 
B: $g(x)=100x$, 
C: $g(x)=-10x$, 
D: approximate $f=10x$ directly.    

\end{example}

\begin{figure}[htb]
\centering
\includegraphics[width=1.0\textwidth,trim=50 5 50 20,clip]{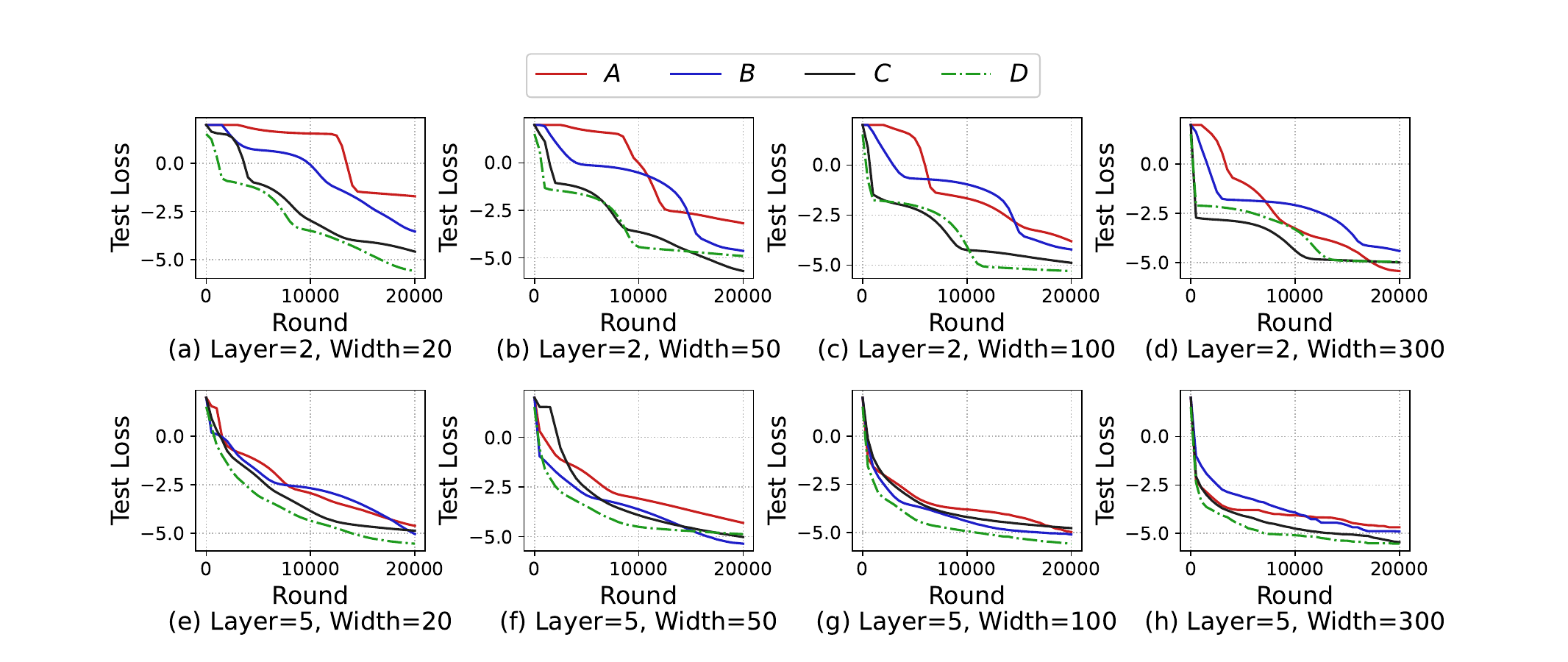}     
\caption{{\color{black}The approximation performance of different pre-training strategies.} The legends $A$, $B$, and $C$ represent the results where the neural network first approximates the {\color{black}intermediate} objective functions $g=-100x$, $g=100x$, and $g=-10x$, respectively, before subsequently approximating the objective function $f$. Legend $D$ corresponds to the scenario in which the neural network directly approximates the linear function $f$. The optimizer used is \texttt{Adam}, and the loss function is mean squared error (MSE). The learning rate is  \(10^{-3}\).}
\label{fig:linear4}
\end{figure}

{Fig.~\ref{fig:linear4} shows the training results for all training strategies. It can be seen that the direct approximation method has the lowest final test error in five initializations (($a$),($c$),($e$),($g$), and ($h$)). Among the six initialization strategies, Strategy A resulted in the largest final error (($a$),($b$),($c$),($e$),($f$), and ($h$)). These experimental results indicate that the greater the difference in slope between the intermediate objective function and the final objective function, the more difficult the second approximation becomes.
This suggests that pre-training is not always effective, and it is necessary to distinguish the impact of different pre-training objective functions on the final outcome. As discussed earlier, the slope can be regarded as a special case of the \texttt{VC}. In the following two subsections, we will build upon the concept of \texttt{VC} to define a novel \texttt{VC} norm and a \texttt{VC}-based pre-processing algorithm.

\subsection{A new distance for pre-processing}

Based on the previous analysis, this section gives the definition of a novel distance based on \texttt{VC} of different $L$ to observe the gap between the objective function and neural networks in the sense of \texttt{VC}.  
In order to measure the \texttt{VC} under different $L$ conditions (measuring the \texttt{VC} of functions at different scales), we first provide the definition of \texttt{IVC}.

\begin{definition}[\texttt{Integral VC (IVC)}]\label{IVC} 
{\color{black} 
The \texttt{IVC} of $f: \mathbb{R}^n \rightarrow\mathbb{R}$ on $[L_{min},L_{max}]$ is defined as 
\[\texttt{IVC}(f,{\bm x}) = \frac{1}{L_{max}-L_{min}}\int_{L_{min}}^{L_{max}} \texttt{VC}_{L}(f,{\bm x})dL. 
\]}
\end{definition}

Since the \texttt{VC} is related to the ratio $L$, the \texttt{VC} of a function f with different values of L reflects the variation of f across different local regions. The integral \texttt{VC} (\texttt{IVC}) provides a comprehensive evaluation of the \texttt{VC} over varying ratios of $L$. According to the definition of \texttt{IVC}, the \texttt{IVC} distance can be naturally defined.

\begin{definition}[\texttt{IVC} distance]\label{IVCD}  Suppose \(f:\mathbb{R}^n \rightarrow\mathbb{R}\), the \texttt{IVC} distance between the different functions \(f_1\) and \(f_2\) is defined as  
\(\texttt{Dist}_{\texttt{IVC}}(f_1,f_2)=\int_{\Omega}\texttt{IVC}(f_1-f_2,{\bm x})d{\bm x}.
\) 
\end{definition}
 
\begin{theorem}
\texttt{IVC} distance satisfies four fundamental properties of distance:

\begin{enumerate}
\item $0\leq \texttt{Dist}_{\texttt{IVC}}(f_1,f_2)<\infty$ for all $f_1$ and $f_2$ in the function space.

\item $\texttt{Dist}_{\texttt{IVC}}(f_1,f_2)=0$ if and only if $f_1=f_2+C$, where $C$ is a constant.

\item $\texttt{Dist}_{\texttt{IVC}}(f_1,f_2)=\texttt{Dist}_{\texttt{IVC}}(f_2,f_1)$ for all $f_1$ and $f_2$.

\item $\texttt{Dist}_{\texttt{IVC}}(f_1,f_3)\leq \texttt{Dist}_{\texttt{IVC}}(f_1,f_2)+\texttt{Dist}_{\texttt{IVC}}(f_2,f_3)$ for all $f_1, f_2, f_3$.
\end{enumerate}

\end{theorem}

\begin{proof}

1. The correctness of this property can be obviously obtained through the definition of \texttt{VC}. 
2. It is obvious that $\texttt{VC}_L(f_1-f_2,{\bm x})$ and  $\texttt{IVC}(f_1-f_2,{\bm x})$ are equal to 0 when $f_1=f_2$. In the case where $\texttt{Dist}_{\texttt{IVC}}(f_1,f_2)=0$, we have 
\[
\texttt{Dist}_{\texttt{IVC}}(f_1,f_2)=\frac{1}{L_{max}-L_{min}}\int_{\Omega}\int_{L_{min}}^{L_{max}} \texttt{VC}_{L}(f_1-f_2,{\bm x})dLd{\bm x}=0.
\]
Since $\texttt{VC}_{L}$ is non-negative, $\texttt{VC}_{L}(f_1-f_2,{\bm x})=0$. From the definition of $\texttt{VC}$, $\forall\ {\bm x}\in\Omega$, we have $f_1({\bm x})=f_2({\bm x})$. 
3. From the definition of $\texttt{VC}$, commutativity of \texttt{IVC} distance is established. 
4. $\forall\ {\bm y}_1, {\bm y}_2 \in \Omega \cap \prod_{i=1}^{n} [{x}_i-L/2, {x}_i+L/2]$, we have
\[\begin{aligned}
|f_1({\bm y}_1)-f_3({\bm y}_1)-(f_1({\bm y}_2)-f_3({\bm y}_2))|
&\leq|f_1({\bm y}_1)-f_2({\bm y}_1)-(f_1({\bm y}_2)-f_2({\bm y}_2))|\\
&+|f_2({\bm y}_1)-f_3({\bm y}_1)-(f_2({\bm y}_2)-f_3({\bm y}_2))|. 
\end{aligned}\] 
Both sides take the supremum of ${\bm y}_1$ and ${\bm y}_2$ at the same time and then we can obtain that 
\[\texttt{VC}_L(f_1-f_3,{\bm x})\leq\texttt{VC}_L(f_1-f_2,{\bm x})+\texttt{VC}_L(f_2-f_3,{\bm x}).
\]
From the definitions of $\texttt{Dist}_{\texttt{IVC}}$ and \texttt{IVC}, the triangle inequality of \texttt{IVC} distance can be obtained.
\end{proof}

\begin{sloppypar}
{In Ex. \ref{eg3}, the pre-processed functions are approximately equal to (A) $-100x$, (B) $100x$, (C) $-10x$, and (D) $0$, respectively. {\color{black}This means that in \([a+L/2,b-L/2]\), \(\texttt{VC}_L(f-\psi_{pre},{\bm x})=110L (A), 90L (B),\) \( 20L(C), 10L(D)\).}  Therefore, under the same $L$ interval and spatial region $\Omega$, $\texttt{Dist}_{\texttt{IVC}}(f,\psi_{pre,A})>\texttt{Dist}_{\texttt{IVC}}(f,\psi_{pre,B})>\texttt{Dist}_{\texttt{IVC}}(f,\psi_{pre,C})>\texttt{Dist}_{\texttt{IVC}}(f,\psi_{pre,D})$. This sequence happens to be similar to the actual training results, which prompts us to design pre-processing (pre-processing) algorithms for neural networks based on $\texttt{Dist}_{\texttt{IVC}}$.}
\end{sloppypar}

\subsection{An accelerated framework for pre-processed neural network approximation}

The above results indicate that \texttt{VC} is related to the convergence of neural networks. To improve approximation, we propose \texttt{VC} to tune the neural network to improve approximation and save training costs. We compare the pre-processing accelerated methods with the original neural network approximation (without pre-processing) in numerical experiments.

\begin{algorithm}
\caption{Neural Network Approximation with Pre-processing Under {\ttfamily IVC}}\label{Alg1}
\begin{algorithmic}[1] \small 
\Require Sampling domain \( \Omega \), sample points \( \{\bm{x}_i\}_{i=1}^P \subset \Omega \), option flag \( \texttt{mode} \in \{\texttt{NN}, \texttt{SUR}\} \), pre-processing threshold \( \varepsilon > 0 \) 

\noindent \hspace{-0.7cm} {\bf Output:} Trained neural network model \( \psi \)

\State Estimate the {\ttfamily IVC} complexity of \( f \) from \( \{\bm{x}_i\} \), denoted as \( \widehat{\texttt{VC}} \)
 
\If{\( \texttt{mode} = \texttt{NN}\) (\texttt{VCP-NN})} 
    \State Design a compact neural network \( \psi_{\text{pre}} \) such that its \texttt{VC}  approximates \( \widehat{\texttt{VC}} \): {\color{black}for instance, }
    Train \( \psi_{\text{pre}} \) on \( \{(\bm{x}_i, f(\bm{x}_i))\}_{i=1}^P \) using an IVC-based loss {\color{black}or Least squares Loss} until 
    \(
    \texttt{Dist}_{\texttt{IVC}}(\psi_{\text{pre}}, f) \leq \varepsilon
    \) 
    \State Expand \( \psi_{\text{pre}} \) into a larger network \( \psi_0 \) by inserting additional inactive neurons (e.g., with zero weights) in the first layer, such that \( \psi_0(\bm{x}) = \psi_{\text{pre}}(\bm{x}) \) for all \( \bm{x} \)
    \State Initialize a trainable neural network \( \phi \leftarrow \psi_0 \)
    \State Train \( \phi \) on \( \{(\bm{x}_i, f(\bm{x}_i))\}_{i=1}^P \) using a user-specified loss function \( \mathcal{L}(\phi(\bm{x}), f(\bm{x})) \)
    \State Set final model \( \psi = \phi \)

\ElsIf{\( \texttt{mode} = \texttt{SUR} \) (\texttt{VCP-obj})} 
    \State Construct a surrogate function \( f_{\text{appr}} \) (e.g., via polynomial or other training-free interpolation or regression) such that
    \(
    \texttt{Dist}_{\texttt{IVC}}(f_{\text{appr}}, f) \leq \varepsilon
    \)
    \State Initialize a neural network \( \phi\) with random weights
    \State Train \( \phi \) on \( \{(\bm{x}_i, f(\bm{x}_i) - f_{\text{appr}}(\bm{x}_i))\}_{i=1}^P \) using user-specified loss \( \mathcal{L}(\phi(\bm{x}), f(\bm{x}) - f_{\text{appr}}(\bm{x})) \)
    \State Set final model \( \psi(\bm{x}) = \phi(\bm{x}) + f_{\text{appr}}(\bm{x}) \)

\EndIf

\State \Return \( \psi \)
\end{algorithmic}
\end{algorithm}

 The core idea of the pre-processing algorithm (Algorithm \ref{Alg1}) is to leverage pre-processing or pre-training techniques to transform the neural network such that it approximates the objective function in terms of \texttt{IVC} distance. This transformation aims to simplify the subsequent training process, either by reducing training time or improving accuracy. Depending on the role of the \texttt{IVC} distance, the algorithm can be divided into two distinct frameworks.

In the first framework, the goal is to design a neural network pretraining process such that the resulting network approximates the objective function in terms of \texttt{IVC} distance. The pretraining process consists of two main stages: training and expansion. Initially, a compact neural network is trained using the \texttt{IVC} distance as the loss function to approximate the shape of the objective function. Subsequently, the compact network is expanded to obtain a preprocessed neural network. This preprocessed model is then fine-tuned using conventional training methods to further approximate the objective function.
In the second framework, the objective is to construct a surrogate function $f_{\rm appr}$ that approximates the objective function in terms of \texttt{IVC} distance. This surrogate function serves as a pre-processing model. A new neural network is then trained to approximate $f-f_{\rm appr}$, thereby indirectly approximating the original objective function. 
{\color{black} The flow past a cylinder example in Section \ref{cylinder} is used to demonstrate the impact of the interpolation.
Fig.~\ref{VCDvsL2} presents a comparison between the $L^2$ distance and the $\texttt{Dist}_{\texttt{IVC}}$ between the surrogate function and the objective function under varying numbers of interpolation points. At lower numbers of interpolation points, the $L^2$ distance between the interpolating and objective functions exhibits a more pronounced reduction; with increasing numbers of points, the $\texttt{Dist}_{\texttt{IVC}}$ also declines rapidly.}

\begin{figure}
    \centering
    \includegraphics[width=0.6\linewidth]{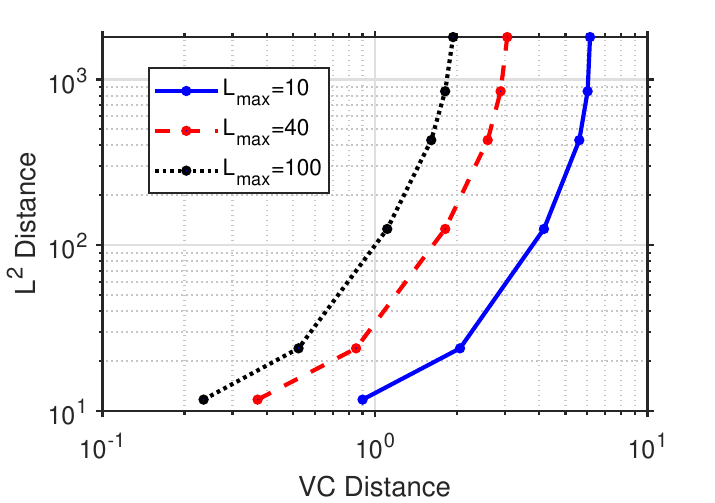}
    \caption{Synchronous decrease of $L^2$ error and \texttt{VC}-{\color{black}norm} error with varying numbers of interpolation points. }
    \label{VCDvsL2}
\end{figure}

\section{Numerical results}\label{sec6}

\subsection{Pre-processing on three-dimensional linear function approximation}

To begin, we investigate the effectiveness of the pre-processing algorithm using a simple linear function approximation task. The experiment compares four approaches: the original approximation algorithm, pre-processing with \texttt{VCP-NN}, pre-processing with \texttt{VCP-obj}, and a modified version of \texttt{VCP-obj}. To maintain consistency, the objective function is chosen to match the example in Section \ref{eg_lin}: a three-dimensional linear function $f(x,y,z)=10x + 10y + 10z$ defined on the region $[-1, 1]\times[-1, 1]\times[-1, 1]$. 
{We design four groups of experiments (i.e. A, B, C, and D), and the experimental settings are as follows}:  A: We first utilize the neural network to approximate the function $5x+5y+5z$. Then the pre-processed neural network approximates the function $f=10x+10y+10z$. B: We approximate the function $f=10x+10y+10z$ directly. C: Define a pre-adjust function $g_n$, and let neural network+$g_{n}$ approximate $f=10x+10y+10z$. The neural network is learnable and the function $g_n$ is frozen.  D: Define a pre-adjust function $g_n$, and let neural network +$g_n/2$ approximate $f=10x+10y+10z$. The neural network is learnable and the function $g_n$ is frozen. 
Fig.~\ref{fig:adjust-model-2} exhibits the loss curves of the above pre-processing methods and original neural network methods. It can be observed that all three A, C, D methods achieve faster reductions in test error compared to the original direct approximation approach under various neural network configurations. Specifically, method A corresponds to a variant of \texttt{VCP-NN} in the pre-processing algorithm. When approximating the objective function $f(x,y,z)$, this approach yields a pre-trained model $\psi_{pre}$ satisfying $\texttt{Dist}_{\texttt{IVC}}(\psi_{\text{pre}},f)\lesssim\texttt{Dist}_{\texttt{IVC}}(\psi_{nn},f)$. The neural network with the smaller \texttt{VC}-distance is thus able to approximate the objective function more rapidly than direct training from scratch. 
Methods C and D correspond to the \texttt{VCP-obj} operation in the pre-processing algorithm, where the auxiliary functions $g_n$ are obtained via an interpolation scheme. The difference lies in the neural network training target: in method C, the network approximates $f-g_n$, whereas in method D, it approximates $f-g_n/2$. This distinction leads to a smaller \texttt{VC}-distance in method C, i.e., $\texttt{Dist}_{\texttt{IVC}}(f_{nn}+g_n,f)\lesssim\texttt{Dist}_{\texttt{IVC}}(f_{nn}+g_n/2,f)$. Experimental results also show that the test error during training in method C is significantly lower than that in method D.

\begin{figure}[htbp]
\centering
\includegraphics[width=1.0\textwidth,trim=50 5 50 20,clip]{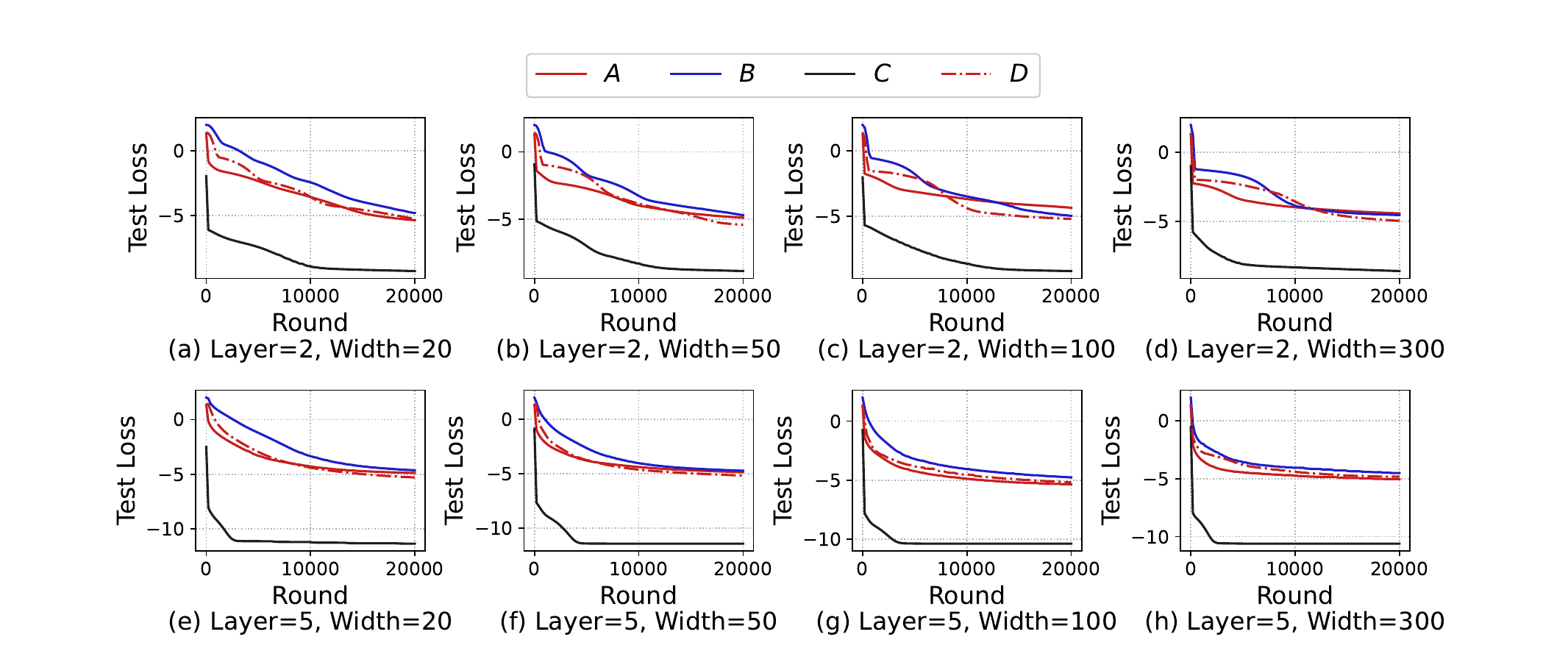}     
\caption{{\color{black}Pre-processing for linear function approximation.} The relation of test loss and round when adjusting the initial model and modeling $f=10x+10y+10z$ with a domain $[-1,1]^3$ under the logarithmic vertical axis scale. The optimizer is \texttt{Adam} and the loss is \texttt{MSE}. The learning rate is  \(10^{-2}\).}
\label{fig:adjust-model-2}
\end{figure}

\begin{figure}[htb]
\centering
\includegraphics[width=1.0\textwidth,trim=50 5 50 20,clip]{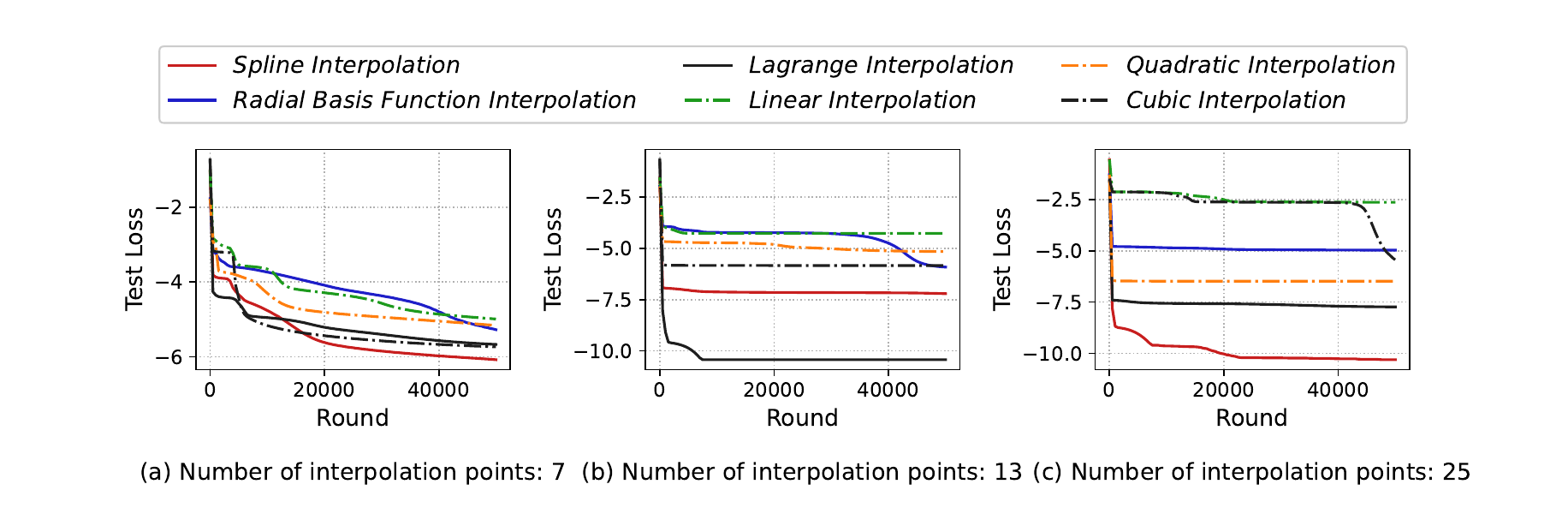}     
\caption{Relation of test loss and round when fitting  $sin(x)-g$ under the logarithmic vertical axis scale with layer number set to 5 and width set to 50. The function $g$ is the interpolation function. The optimizer is \texttt{Adam} and the loss is MSE. The learning rate is  \(10^{-3}\). The number of interpolation points is selected from \{7, 13, 25\}.}
\label{fig:sin2}
\end{figure}

\subsection{Pre-processing on monochrome image approximation}

We examine how pre-processing contributes to improving performance in two-dimensional image approximation tasks. In the process of function approximation for images, we adopt three different approximation strategies. The basic approach is the direct approximation using a neural network. The second strategy employs the \texttt{VCP-NN} method, in which we first divide the objective function by $2$ and approximate this simplified model using a neural network, thereby reducing the VC distance between the NN and the objective function. This is then compared with the direct approximation approach. The third strategy utilizes the \texttt{VCP-obj} method, where a preprocessed interpolated surrogate model is first used to approximate the objective function, followed by neural network approximation of the difference $f-f_{\rm appr}$. Here we use fully connected neural networks with $5$ layers and either $50$ or $300$ neurons per layer.

\begin{figure}[htb]
\centering
\includegraphics[width=1.0\textwidth,trim=50 0 0 0,clip]{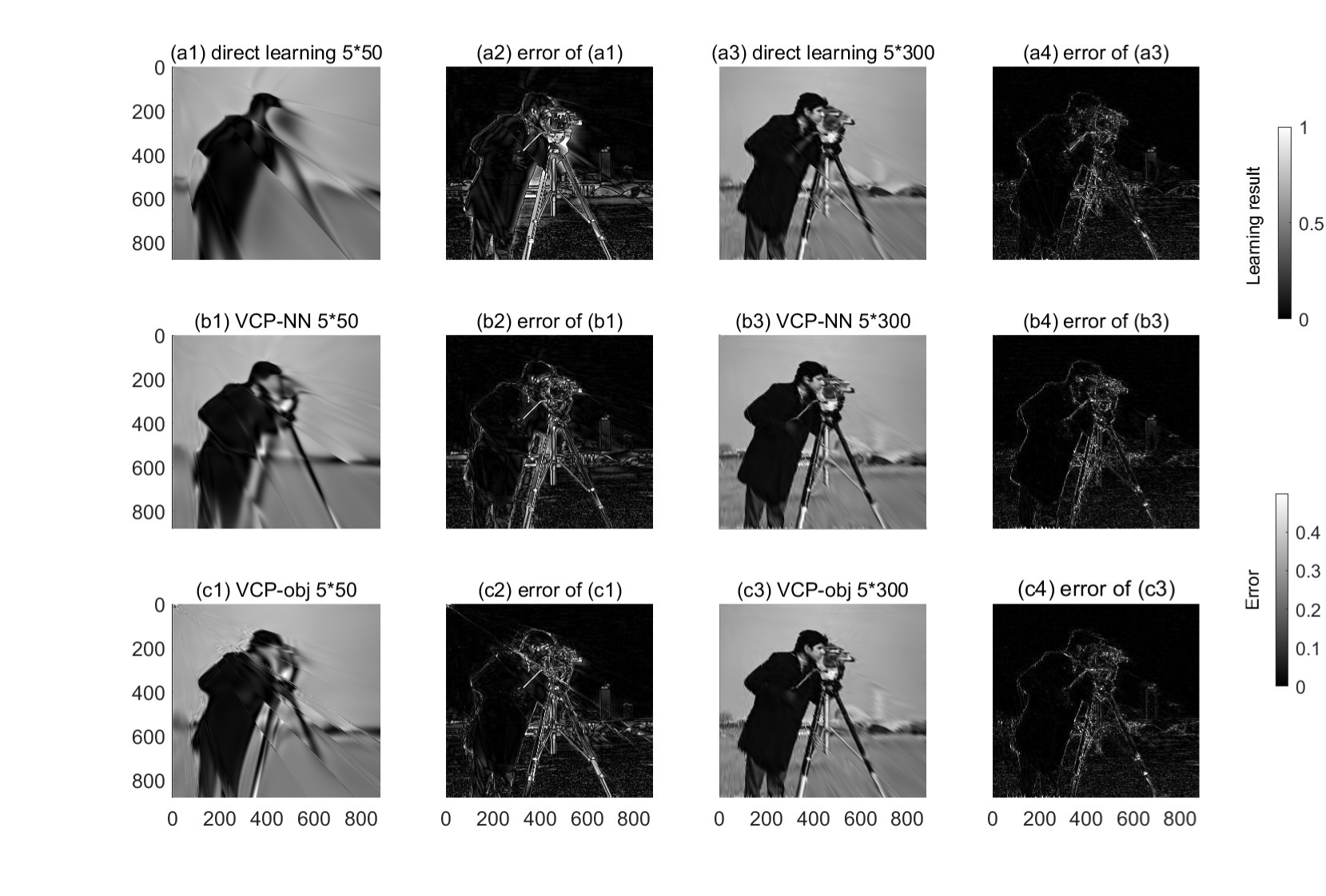}   
\caption{{\color{black}Pre-processing for image approximation.} The 1st and 3rd columns show the training results of three methods for approximating monochrome pictures in two neural network settings. The other columns show the final absolute error.}
\label{fig:compare}
\end{figure}

\begin{figure}[htb]
\centering
\includegraphics[width=0.6\textwidth,trim=0 0 0 0,clip]{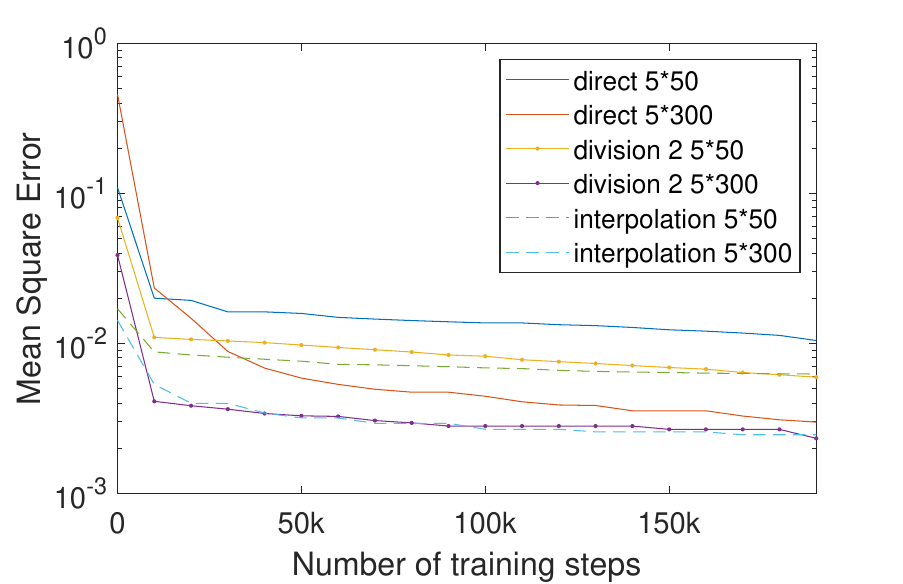}   
\caption{The six curves represent the decreasing trend of test absolute error for three training strategies under two neural network settings.}
\label{fig:compare1}
\end{figure} 

Fig.~\ref{fig:compare} presents the experimental results of image approximation. The two leftmost columns indicate that, under a FCNN with architecture $5\times50$, both preprocessing strategies can effectively improve the approximation performance. However, due to the limited capacity of the network, significant errors remain. After increasing the network width, the approximation performance of all three strategies improves markedly. Fig.~\ref{fig:compare1} shows the evolution of the test error over the course of training for all six experiments corresponding to the three strategies. The results demonstrate that the test errors of the two preprocessing-based methods (represented by dashed and dotted lines) are significantly lower than those of the original method. This indicates that the two preprocessing approaches can either reduce the training time required to achieve a given level of accuracy or improve the final accuracy within the same training time.

\subsection{Pre-processing for 2-dimensional flow past a cylinder}

In this section, we employ the previously introduced example of cylinder flow to assess the effectiveness of the proposed VCP algorithm in scientific computing, using numerical results reported in \cite{trebotich2015adaptive}. Specifically, we adopt the \texttt{VCP-obj} preprocessing strategy, in which a surrogate model is first constructed to approximate the objective data in terms of the \texttt{IVC} distance. A neural network is then trained to approximate the discrepancy between the surrogate model and the objective data, and the final prediction of the objective data is obtained by combining the outputs of the surrogate model and the neural network.

\begin{figure}[htb]
    \centering
 
    \begin{subfigure}{0.23\textwidth}
        \includegraphics[width=\linewidth, trim=30 15 15 31, clip]{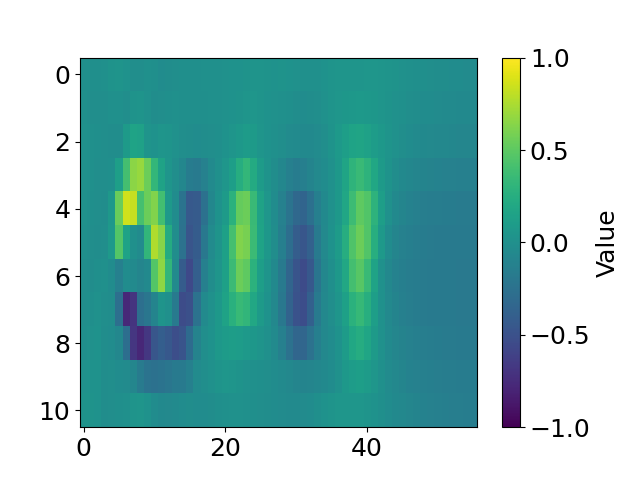} 
    \end{subfigure}
    \begin{subfigure}{0.23\textwidth}
        \includegraphics[width=\linewidth, trim=30 15 15 31, clip]{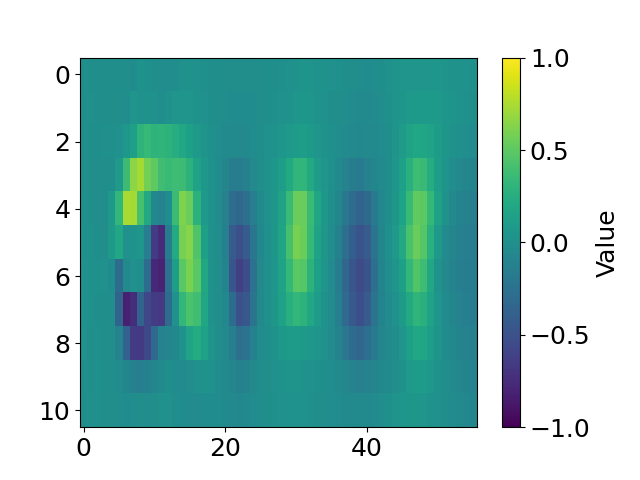} 
    \end{subfigure}
    \begin{subfigure}{0.23\textwidth}
        \includegraphics[width=\linewidth, trim=30 15 15 31, clip]{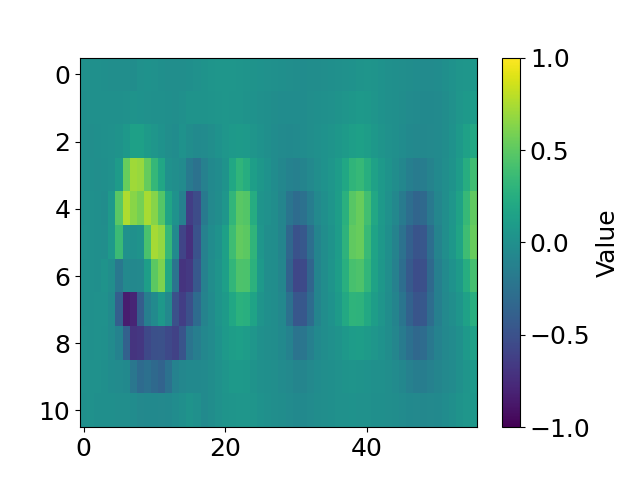} 
    \end{subfigure}
    \begin{subfigure}{0.23\textwidth}
        \includegraphics[width=\linewidth, trim=30 15 15 31, clip]{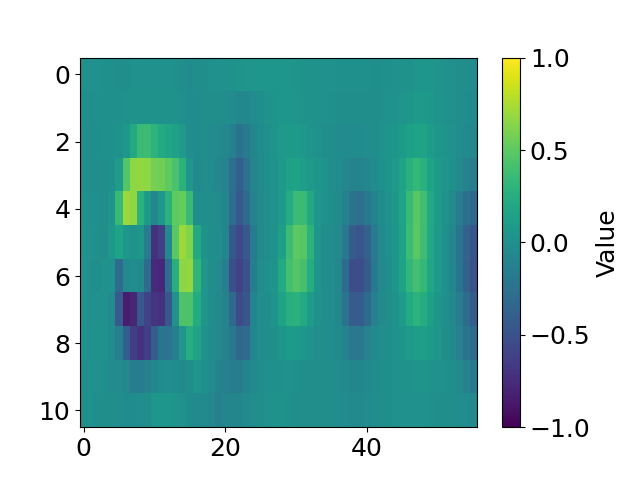} 
    \end{subfigure}
 
    \begin{subfigure}{0.23\textwidth}
        \includegraphics[width=\linewidth, trim=30 15 15 31, clip]{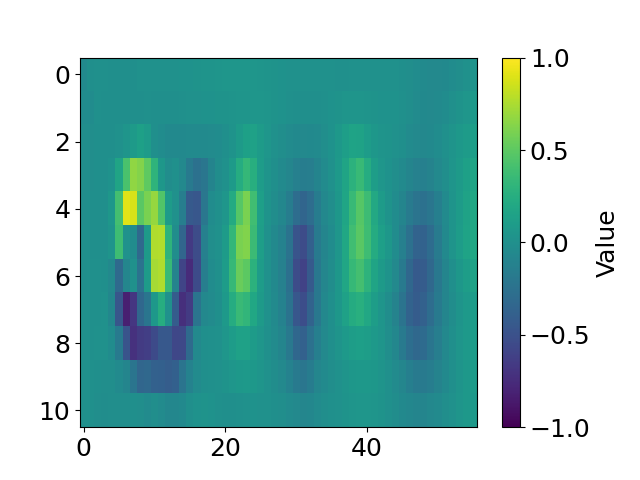} 
    \end{subfigure}
    \begin{subfigure}{0.23\textwidth}
        \includegraphics[width=\linewidth, trim=30 15 15 31, clip]{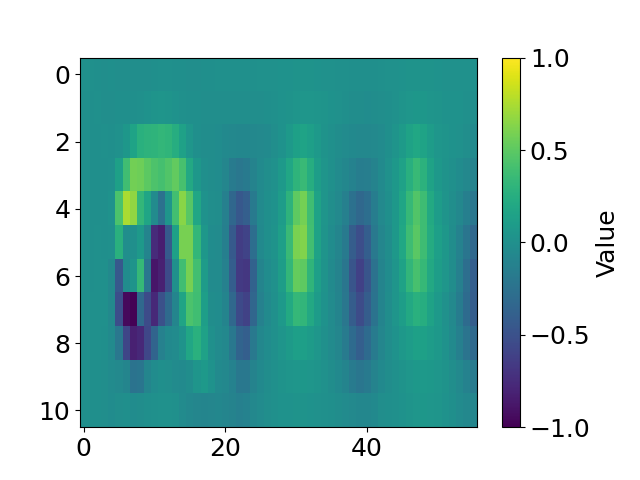} 
    \end{subfigure}
    \begin{subfigure}{0.23\textwidth}
        \includegraphics[width=\linewidth, trim=30 15 15 31, clip]{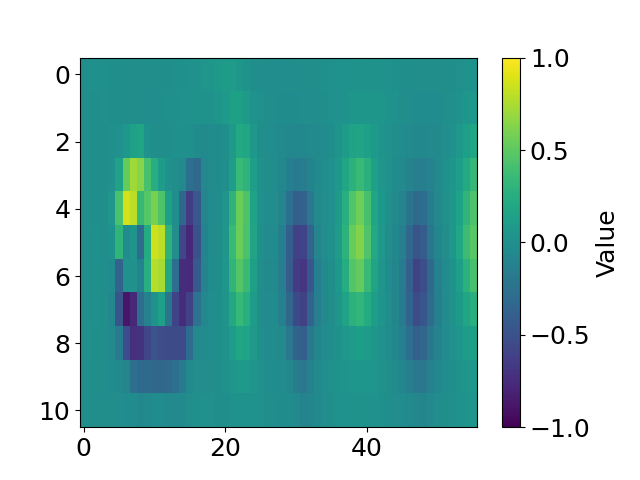} 
    \end{subfigure}
    \begin{subfigure}{0.23\textwidth}
        \includegraphics[width=\linewidth, trim=30 15 15 31, clip]{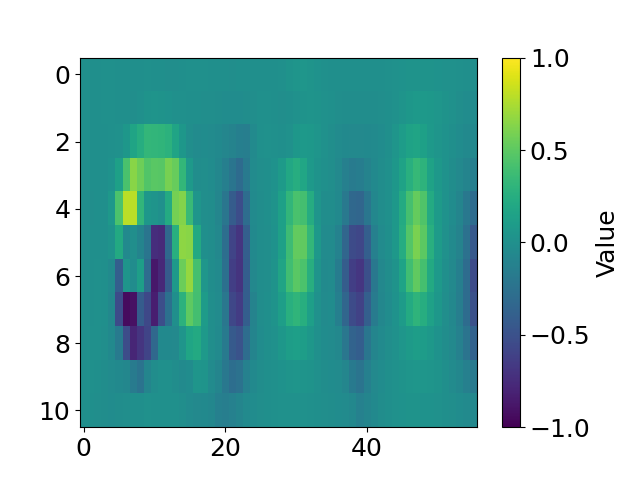} 
    \end{subfigure}

    \begin{subfigure}{0.23\textwidth}
        \includegraphics[width=\linewidth, trim=30 15 15 31, clip]{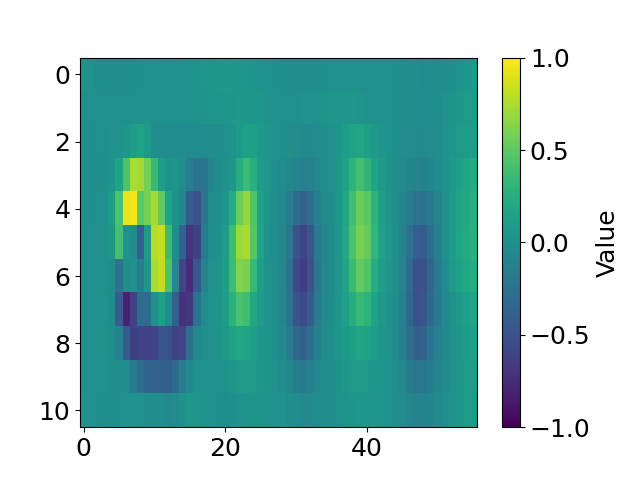}
        \caption*{$t=0$}
    \end{subfigure}
    \begin{subfigure}{0.23\textwidth}
        \includegraphics[width=\linewidth, trim=30 15 15 31, clip]{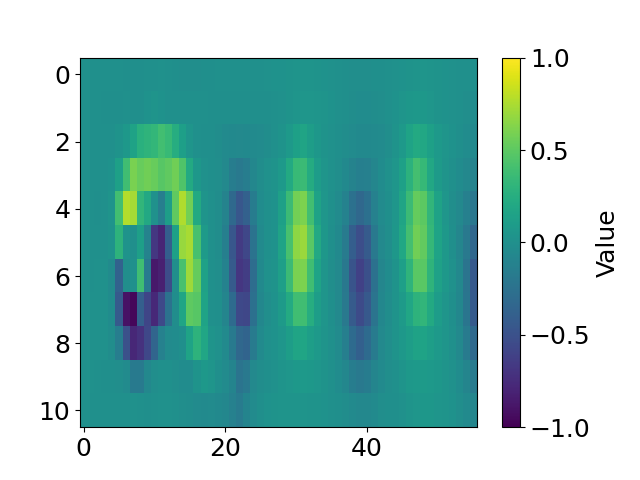}
        \caption*{$t=3$}
    \end{subfigure}
    \begin{subfigure}{0.23\textwidth}
        \includegraphics[width=\linewidth, trim=30 15 15 31, clip]{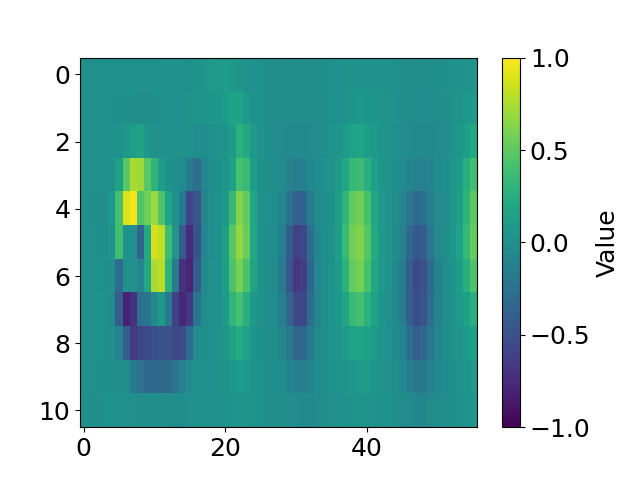}
        \caption*{$t=6$}
    \end{subfigure}
    \begin{subfigure}{0.23\textwidth}
        \includegraphics[width=\linewidth, trim=30 15 15 31, clip]{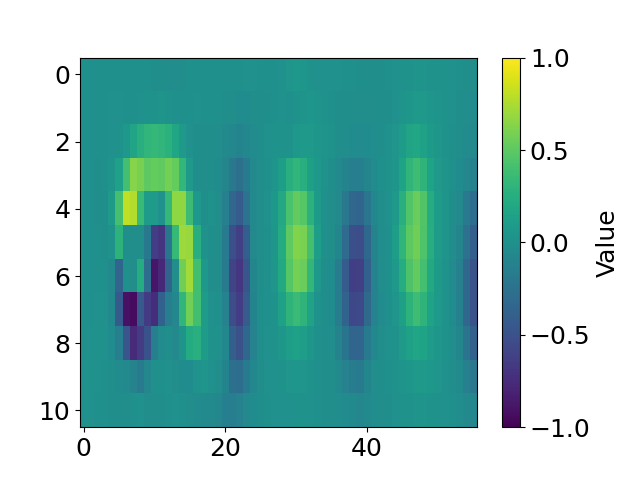}
        \caption*{$t=9$}
    \end{subfigure}

    \caption{Approximation results across different time steps and pre-processings: rows indicate increasing levels of pre-processing (none, brief, sufficient), columns show time steps.}
    \label{fig:fit_results_pde}
\end{figure}

This experiment employs a fully connected neural network (FCNN) on 2-dimensional flow past a cylinder data, with the following structural and hyperparameter specifications: The input layer dimension is set to 3, corresponding to the feature vector dimensionality of the 2-dimensional flow past a cylinder data, while the output layer is fixed to 1. The model adopts a deep architecture comprising 8 hidden layers, each uniformly configured with 50 neurons to hierarchically extract high-dimensional spatiotemporal patterns. Optimization is performed using Adam with a learning rate of 0.01 to accelerate convergence. A batch size of 512 is selected. The training process spans 30,000 epochs to ensure comprehensive parameter space exploration. For dynamic monitoring, loss values are recorded at intervals of 100 epochs to facilitate convergence analysis and training dynamics visualization.  
Fig.~\ref{fig:fit_results_pde} presents the results of three experiments—direct approximation, brief \texttt{VCP-obj} approximation (interpolation-based pre-processing with partial data points), and sufficient \texttt{VCP-obj} approximation (interpolation-based pre-processing with all data points). The differences among the three approaches are most evident in the downstream region of the cylinder flow. In the case of direct approximation, the neural network fails to capture the downstream behavior, suggesting that the poor performance is likely due to the limited approximation capacity of the network, despite sufficient training. By contrast, when the same neural network is applied with interpolation-based preprocessing using partial data points, the approximation improves substantially. This indicates that the interpolation-based preprocessing method (i.e., the \texttt{VCP-obj} approach) can effectively compensate for the network’s insufficient approximation capacity in certain problems. Moreover, using all interpolation points yields results similar to those obtained with partial interpolation, suggesting that in practice only a subset of data points is sufficient for the \texttt{VCP-obj} method.

\section{Conclusions}\label{sec7}

In this paper, we proposed the \texttt{VC} metric, a novel approach to {evaluate} local performance and behavior in neural network approximations. Such a new metric addresses the challenge of unpredictable local performance in neural networks, providing a quantifiable measure to enhance the assessment stability and guide the improvement of network design and training. We explored the basic properties of \texttt{VC}, proved its rationality, and showed its usefulness in practical applications. A minority-tendency phenomenon in neural network approximation is found. Our theoretical contributions also include the introduction of the \texttt{VC density}, providing a probabilistic view of function-value changes in function behavior. We established key properties of \texttt{VC}, such as its affine invariance, demonstrating the robustness of the metric across various transformations. We give a pre-processing framework for the neural network approximation.

{\bf Acknowledgments.} 
The authors would like to thank Dr. Stefan M. Wild (Lawrence Berkeley National Laboratory) for his valuable suggestions and guidance. We are also grateful to Dr. David Trebotich (Lawrence Berkeley National Laboratory) for his assistance with the implementation and codes for the example ``2-dimensional flow past a cylinder.'' This work was done before the corresponding author joined LBNL.

\bibliographystyle{siamplain}
\bibliography{references}

\end{document}